\documentclass{article}
\usepackage{amsmath,amsfonts,bm}

\def\eqref#1{equation~\ref{#1}}
\def\1{\bm{1}}

\DeclareMathAlphabet{\mathsfit}{\encodingdefault}{\sfdefault}{m}{sl}
\SetMathAlphabet{\mathsfit}{bold}{\encodingdefault}{\sfdefault}{bx}{n}

\newcommand{\E}{\mathbb{E}}

\newcommand{\R}{\mathbb{R}}

\usepackage{microtype}
\usepackage{graphicx}
\usepackage{booktabs} % for professional tables

\usepackage{hyperref}

\usepackage[accepted]{icml2023}

\usepackage{amsmath}
\usepackage{amssymb}
\usepackage{mathtools}
\usepackage{amsthm}

\usepackage{url}
\usepackage{algorithm}
\usepackage{algorithmic}
\usepackage{graphicx}
\usepackage{subfig}
\usepackage{wrapfig}
\usepackage{caption}
\usepackage[export]{adjustbox}

\theoremstyle{definition}

\newcommand{\milr}{\eta^{\ast}}

\newtheorem{theorem}{Theorem}
\newtheorem{lemma}{Lemma}
\newtheorem{corollary}{Corollary}
\newcommand{\Y}[2]{Y^{(#1)}\left[#2\right]}
\newcommand{\lrr}[1]{\left(#1\right)}
\newcommand{\Ee}[1]{\mathbb E \left[#1\right]}
\newcommand{\norm}[1]{\left|\left|#1\right|\right|}
\newcommand{\gives}{\rightarrow}
\newcommand{\abs}[1]{\left|#1\right|}
\newcommand{\set}[1]{\left\{#1\right\}}

\usepackage[textsize=tiny]{todonotes}

\begin{document}

\twocolumn[
\icmltitle{Maximal Initial Learning Rates in Deep ReLU Networks}




\begin{icmlauthorlist}
\icmlauthor{Gaurav Iyer}{mcgill,mila}
\icmlauthor{Boris Hanin}{princeton}
\icmlauthor{David Rolnick}{mcgill,mila}
\end{icmlauthorlist}

\icmlaffiliation{mcgill}{School of Computer Science, McGill University, Montreal, Canada}
\icmlaffiliation{princeton}{Dept.~of Operations Research \& Financial Engineering, Princeton University, Princeton, USA}
\icmlaffiliation{mila}{Mila -- Quebec AI Institute, Montreal, Canada}

\icmlcorrespondingauthor{Gaurav Iyer}{gaurav.iyer@mila.quebec}

\icmlkeywords{Deep Learning Theory, Learning Rate, Large Learning Rate, Sharpness, Wide Networks, Network Initialization, Depth to Width Ratio, Edge of Stability}

\vskip 0.3in
]



\printAffiliationsAndNotice{}%\icmlEqualContribution} % otherwise use the standard text.

\begin{abstract}
Training a neural network requires choosing a suitable learning rate, which involves a trade-off between speed and effectiveness of convergence. While there has been considerable theoretical and empirical analysis of how large the learning rate can be, most prior work focuses only on late-stage training. In this work, we introduce the \emph{maximal initial learning rate} $\milr$ -- the largest learning rate at which a randomly initialized neural network can successfully begin training and achieve (at least) a given threshold accuracy. Using a simple approach to estimate $\milr$, we observe that in constant-width fully-connected ReLU networks, $\milr$ behaves differently from the maximum learning rate later in training. Specifically, we find that $\milr$ is well predicted as a power of $(\text{depth} \times \text{width})$, provided that (i) the width of the network is sufficiently large compared to the depth, and (ii) the input layer is trained at a relatively small learning rate. We further analyze the relationship between $\milr$ and the sharpness $\lambda_{1}$ of the network at initialization, indicating they are closely though not inversely related. We formally prove bounds for $\lambda_{1}$ in terms of $(\text{depth} \times \text{width})$ that align with our empirical results.
\end{abstract}

\section{Introduction}
\label{intro}

The learning rate plays a crucial role in the training of deep neural networks. Unfortunately, tuning the learning rate is a tricky task -- too large a learning rate can cause the training loss to diverge, while too small a learning rate can result in inefficient use of time and computational resources. The optimal choice of learning rate has been observed to depend non-trivially on many factors, including the data, architecture, optimizer, and initialization scheme. As a result, it can be difficult to theoretically analyze the relationship between the learning rate and other elements of the training framework, and computationally expensive to find the best learning rate in practice.

While there have been numerous works rigorously analyzing effective learning rates both from theoretical and empirical standpoints \citep{NEURIPS2020_a7453a5f,l.2018dont}, these works tend to focus on behavior during late-stage training, considering which learning rates provide optimal performance \citep{pmlr-v97-park19b}, or lead to convergence guarantees \citep{wang2022large}. However, it is not clear that optimal learning rates early in training follow similar patterns to those near convergence, and indeed a variety of heuristics for learning rate scheduling suggest that an early change in learning rate can actually boost performance \citep{goyal2017accurate}.

\begin{figure*}[htb!]
    \centering
    \includegraphics[width=\textwidth]{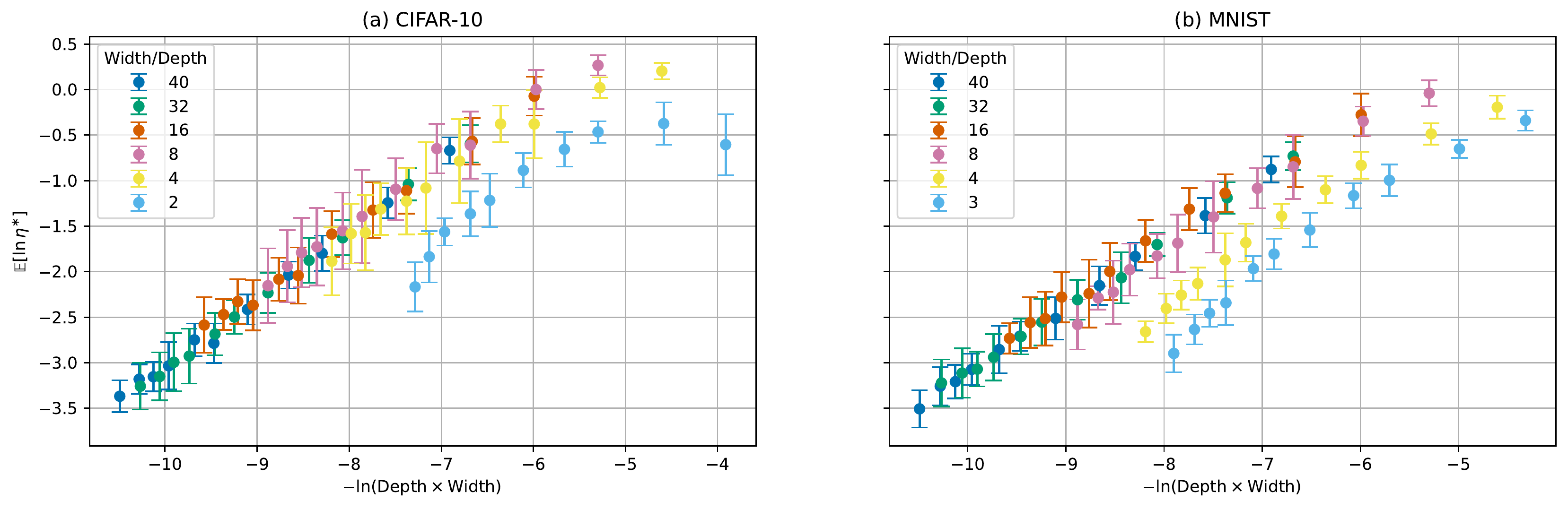}
    \subfloat[\empty]{
        \centering
        \hspace{1cm}
        \begin{tabular}{|c|c|c|}
                \hline
                \textbf{$\text{width} / \text{depth}$} & \textbf{Slope $\alpha$} & \textbf{$R^{2}$}  \\
                \hline \hline
                40 & 0.748 & 0.995\\
                \hline
                32 & 0.76 & 0.999 \\
                \hline
                16 & 0.713 & 0.994 \\
                \hline
                8 & 0.688 & 0.991 \\
                \hline
                4 & 0.587 & 0.967 \\
                \hline
                2 & 0.478 & 0.778 \\
                \hline
            \end{tabular}
        \hspace{3cm}
        \begin{tabular}{|c|c|c|}
                \hline
                \textbf{$\text{width} / \text{depth}$} & \textbf{Slope $\alpha$} & \textbf{$R^{2}$}  \\
                \hline \hline
                40 & 0.72 & 0.998\\
                \hline
                32 & 0.71 & 0.998 \\
                \hline
                16 & 0.687 & 0.991 \\
                \hline
                8 & 0.703 & 0.991 \\
                \hline
                4 & 0.693 & 0.98 \\
                \hline
                3 & 0.714 & 0.968 \\
                \hline
            \end{tabular}  
    }
    \caption{Relationship between maximal initial learning rate $\milr$ and architecture for fully-connected networks with different values of $\text{width} / \text{depth}$. We use depths $\in \{5,7,10,12,15,18,20,23,25,27,30 \}$, for 25 initializations per architecture. If $\text{width} / \text{depth}$ is sufficiently large, the expected $\milr$ displays a strong power relationship with $(\text{depth}\times \text{width})^{-1}$. Interestingly, all such $\milr$ lie on the same line regardless of the exact $\text{width}  / \text{depth}$ value.  As $\text{width} / \text{depth}$ becomes smaller, networks deviate from the power relationship. The input layer of networks is trained at $\eta \cdot 10^{-2}$, where $\eta$ is the learning rate for the rest of the network. We report slopes and coefficient of determination values for each $\text{width} $ to emphasize the linear fit.}
    \label{fig:const_width_depth_smallinputlr}
\end{figure*}

In this paper, we consider both empirically and theoretically how large the learning rate can be in early training. Our main contributions are as follows:
\vspace{-0.1in}
\begin{itemize}
    \item In \S \ref{prelim:milr}, we introduce the \emph{maximal initial learning rate} $\milr$ -- the largest learning rate at which a randomly initialized neural network can successfully begin training -- and show how it can be computed.
    
    \item For fully-connected deep ReLU networks, we empirically identify a
    power law relating the maximal initial learning rate and the product of width and depth: 
    $$\E[\ln\milr] = -\alpha\ln(\text{depth}\times \text{width}) + \gamma_{1},$$
    
    %where $\alpha$ is a constant close to $0.75$ for sufficiently wide networks
    which can be observed in 
    \autoref{fig:const_width_depth_smallinputlr}\footnote{Our empirical results suggest that $\alpha$ is indeed task-dependent. While our results in \autoref{fig:const_width_depth_smallinputlr} largely show $\alpha \approx 0.7$ for sufficiently wide networks, significant differences can be seen in \autoref{fig:gaussian_data}(a), where we obtain smaller $\alpha$ on Gaussian data.}. Notably, while prior theoretical work \citep{pmlr-v89-karakida19a} suggests that at the end of training $\alpha=1$, we find that at the start of training $\alpha < 1$, allowing for larger initial learning rates.
    
    \item We show empirically that the sharpness $\lambda_{1}$ (i.e.~the largest eigenvalue of the training loss Hessian) at initialization is a function of $\milr$, but does not necessarily reflect the inverse relationship observed in late-stage training. We illustrate this in \autoref{fig:sharpness_milr_comparison}. %\bh{Do we find that $\lambda_1$ scales as a constant times $\milr$?? Our slope is like $1.1$ but can we reject the null hypothesis?} 
    \item In \S \ref{sec:thm}, we prove power law upper bounds on $\lambda_{1}$ at initialization as a function of $(\text{depth}\times\text{width})$ by studying the Frobenius norm of the Loss Hessian. 
\end{itemize}

In the process, we also provide a counterexample to a claim about learning rates and sharpness made in \citet{gilmer2022a},
as detailed in \S \ref{section:sharpness}.

\section{Related Work}

\begin{figure*}[!ht]
    \begin{minipage}[c]{0.5\textwidth}
        \includegraphics[width=0.9\textwidth]
        {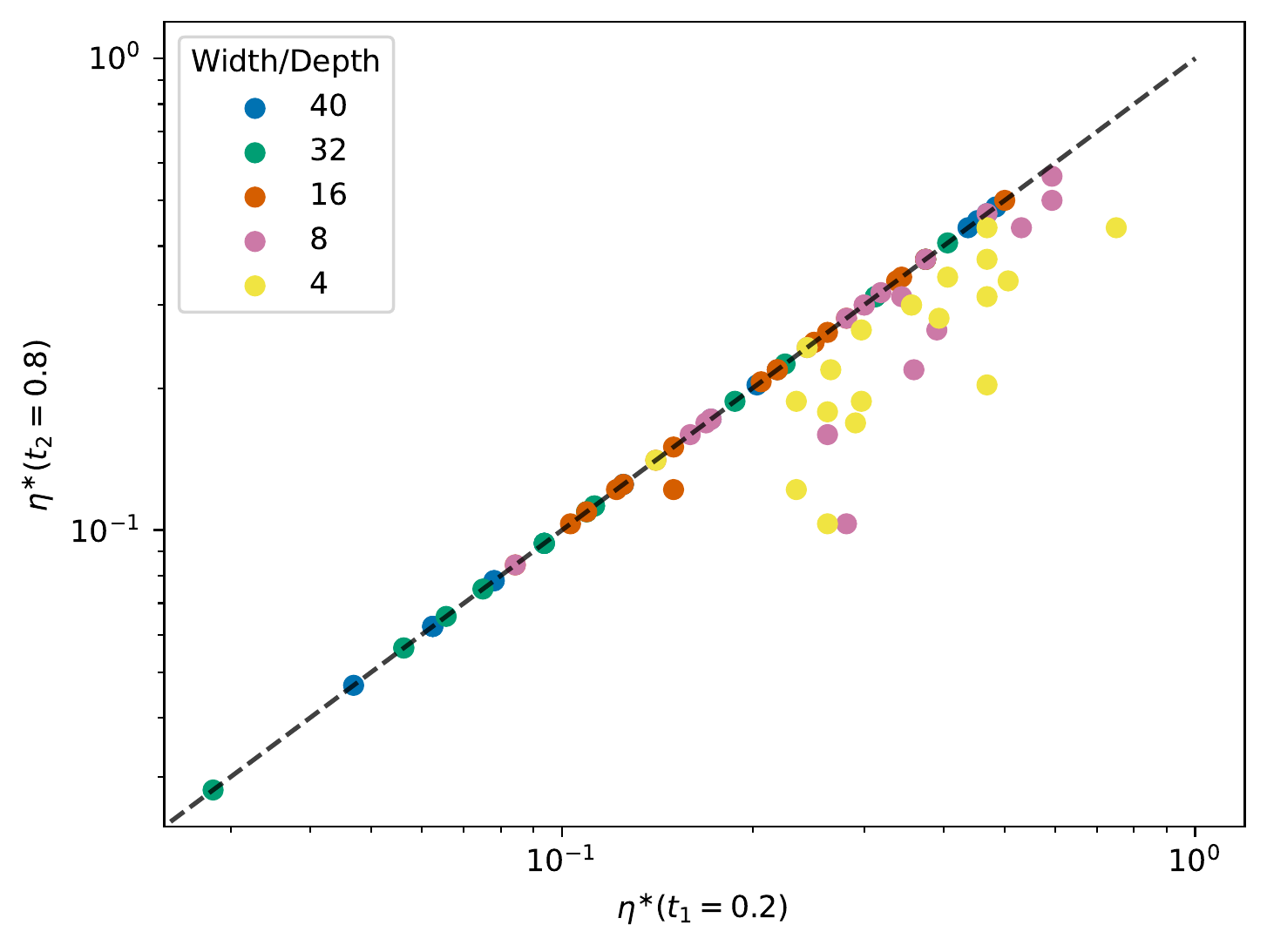}
    \end{minipage} \hfill
    \begin{minipage}[c]{0.5\textwidth}
        \vspace{0.5cm}
        \caption{Behavior of the maximal initial learning rate $\milr$ as network architecture and threshold accuracy $t$ vary in constant-width fully-connected ReLU networks. (a) illustrates how $\milr$ depends on $t$ as a function of $\text{width}  / \text{depth}$ for networks trained on MNIST. We plot $\milr$ at $t=$ 0.2 and 0.8, on the $x$ and $y$ axes respectively, and for 5 initializations per architecture with depths $\in \{5,10,15,20\}$ and $\text{width}  / \text{depth} \in \{4,8,16,32,40\}$. There is no change in $\milr$ for architectures with sufficiently large $\text{width}/ \text{depth}$, and such points fall on the $x=y$ line. Otherwise, its value declines as $t$ becomes larger. 
        Essentially, the value of $\milr$ is stable across a wide range of threshold accuracies $t$ in constant width architectures with sufficiently large width/depth.}
        \label{fig:hero}
    \end{minipage}
\end{figure*}

Large learning rates are a topic of considerable interest \citep{n.2018superconvergence,NEURIPS2019_bce9abf2,jastrzebski2018factors}. For instance
\citet{DBLP:journals/corr/abs-2003-02218}, to which we compare our work more thoroughly below, investigate the benefits of large learning rates to generalization and predict maximum learning rates for some architectures in relation to the sharpness (see below).
Through theoretical analysis of the Fisher information matrix and its statistics, \citet{pmlr-v89-karakida19a} derive a relationship between the architecture of a network and the largest learning rate that would allow the network to converge to the global minimum when trained with SGD. They state that this learning rate should scale linearly with ${(\text{depth}\times \text{width})}$ for constant width fully-connected networks.
\citet{pmlr-v97-park19b} relate the optimal learning rate (in the sense of optimal performance) to the effective width of neural networks by studying the normalized noise scale --  a quantity derived from the learning rate, batch size, and training set size.

\citet{cohen2021gradient} introduce ``progressive sharpening'' and the ``edge of stability'' regime by investigating the evolution of sharpness over the course of training, which is extended to adaptive gradient methods in \citet{https://doi.org/10.48550/arxiv.2207.14484}.
More recently, there have also been several theoretical investigations \citep{pmlr-v162-ahn22a, pmlr-v162-arora22a, https://doi.org/10.48550/arxiv.2207.12678} into this phenomenon.
\citet{gilmer2022a} add to this line of work by looking at early training instabilities through the lens of sharpness and argue that seemingly different methods such as learning rate warmup and gradient clipping stabilize the learning process through the same mechanism -- by reducing sharpness early in training. 
\citet{Jastrzebski2020The} similarly look at the effect of hyperparameters used in the early stages of training and find that they determine properties of the entire training trajectory.

We present an average-case analysis of the maximal initial learning rate, and how it relates to the architecture and expected sharpness. Such average-case analyses can be useful in identifying gaps between theoretical possibilities and practical observations. \citet{hanin2022deep,NEURIPS2019_9766527f} and \citet{pmlr-v97-hanin19a} provide average-case analyses for expected length distortion, the number of linear regions, and the number of activation regions in deep ReLU networks.
For further examples of average-case analyses, see \citet{pmlr-v70-shalev-shwartz17a,pmlr-v70-raghu17a}.

We also point the reader to several prior articles which show that in wide fully-connected networks, it is the width-to-depth ratio, rather than depth or width separately that effectively controls the stability of optimization -- this is the reason we separate architectures based on this ratio in our experimental results. Examples include work about fully-connected networks at initialization, which study the fluctuations of the forward pass \citep{hanin2018start}, the fluctuations in the backward pass \citep{hanin2018neural, hanin2018products, hanin2022correlation}, and the extent of feature learning early in training \citep{hanin2019finite}. Moreover, for large values of width/depth, many novel results of this kind were also obtained for networks \textit{after training} in \citet{roberts2021principles} (especially Chapter $\infty$). 

We conclude the literature review by comparing our maximal initial learning rate to the maximal stable learning rate of \citet{DBLP:journals/corr/abs-2003-02218} in more depth. To start, note that these two notions of maximal learning rate are different. As we illustrate in \autoref{fig:1}(b), our maximal initial learning rate can sometimes lead to instability late in training, suggesting that the maximal initial learning rate is likely larger than the maximal stable learning rate. A core proposal of
\citet{DBLP:journals/corr/abs-2003-02218} is that the maximal stable learning rate has the form $c_{act} / \lambda_1$, where $c_{act}$ is a constant and $\lambda_1$ is the largest eigenvalue of the NTK. For MSE loss and linear one-layer networks, \citet{DBLP:journals/corr/abs-2003-02218} suggest both theoretically and empirically that $c_{act} = 4$. For more general architectures and cross-entropy losses, however, \citet{DBLP:journals/corr/abs-2003-02218} obtains different values of $c_{act}$. Thus, in situations where the maximal initial and maximal stable learning rates are comparable, our empirics can be viewed as capturing more of the full architecture dependence of $c_{act}$, suggesting that perhaps $c_{act} / \lambda_1$ scales like $\text{(depth $\times$ width)}^{-\alpha}$. Finally, as detailed in Appendix A of their work, experiments in \citet{DBLP:journals/corr/abs-2003-02218} sometimes use learning rate drops, cosine scheduling, data augmentation, batch normalization, and weak $L_2$ regularization in experiments. The experiments we undertake in this work do not make use of these, further complicating direct comparisons. 

\section{Preliminaries}

\subsection{The Maximal Initial Learning Rate} \label{prelim:milr}

We define the \emph{maximal initial learning rate} $\milr$ to be the largest, constant learning rate at which a given network can achieve validation accuracy of at least $t$, where $t$ is a given threshold accuracy.

The choice of learning rate is largely dictated by the data being used to train the network and its architecture. Since a theoretical formulation of the data itself is unrealistic, the learning rate needs to be empirically tuned for each problem statement, which is further affected by changes in the training setup. By introducing the maximal initial learning rate $\milr$, exploring its properties, and relating it to the architecture, we aim to make it easier to find large learning rates that work in practice.

Furthermore, several recent lines of work observe that the early phase of training can heavily impact training dynamics and performance at later stages \citep{Frankle2020The,Jastrzebski2020The} -- from this perspective, it is important to consider the behavior of $\milr$ and how it changes as a function of architecture and training setup.

\begin{figure*}[t]
    \centering
    \includegraphics[width=\textwidth]{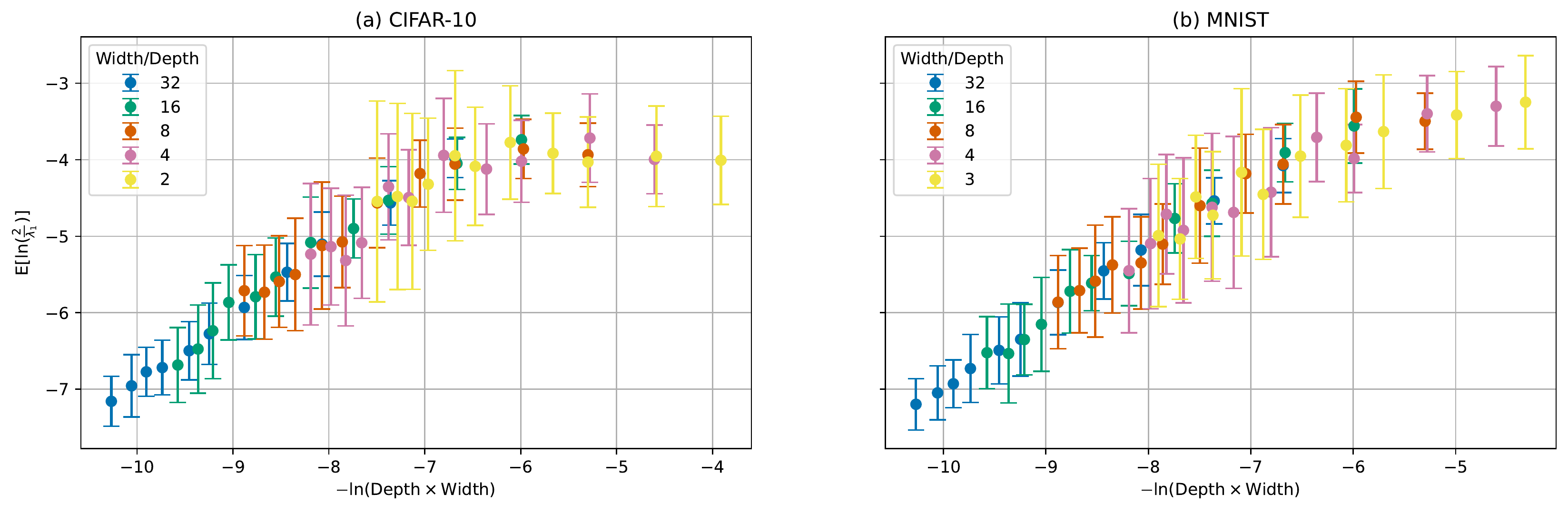}
    \caption{Plots of $2/\lambda_1$ at initialization as a function of $(\text{depth}\times \text{width})$ for various fully-connected, Kaiming-initialized architectures, evaluated on CIFAR-10 (left) and MNIST (right). $\lambda_{1}$ is measured by considering the complete dataset as a single batch of data and is averaged over 25 initializations for each architecture, with error bars shown. See \autoref{fig:const_width_depth_smallinputlr} for architecture depths used. We observe that for sufficiently large $(\text{depth}\times \text{width})$, the maximal initial learning rate $\milr$ and $2/\lambda_1$ both show similar power relationships.}
    \label{fig:sharpness_const_width_depth}
\end{figure*}

\begin{algorithm}[htb!]
\caption{Maximal Initial Learning Rate $\milr$}
\label{alg:MILR}
\begin{algorithmic}

\STATE Define threshold accuracy $t$, and upper and lower learning rates $u$ and $l$ respectively. 
\FOR{a small number of search iterations $s$}
\STATE Compute $m = \frac{u+l}{2}$ i.e.~the midpoint of $u$ and $l$
\FOR{each epoch in a small number of epochs $e$}
\STATE Train network at learning rate $m$
\STATE Evaluate validation accuracy $a$
\STATE If $a \geq t$, then break out of inner loop, and $l \gets m$
\ENDFOR
\STATE If $a < t$ after training for $e$ epochs, then $u \gets m$
\ENDFOR
\STATE The last value of $m$ satisfying $a \geq t$ is the desired $\milr$.

\end{algorithmic}
\end{algorithm}

In \autoref{alg:MILR} we describe a simple method that can be used to approximate $\milr$. We take as input a network initialization, threshold accuracy $t$, and lower and upper learning rates $l$ and $u$ respectively. We perform a binary search on the continuous space of learning rates $\in (l,u)$ to identify $\milr$. More specifically, for each midpoint learning rate $m = \frac{u+l}{2}$, we train the network from initialization for a small number of epochs. If this trained network, at any point during training, achieves validation accuracy at least $t$, then we set $l = m$. Otherwise, we set $u = m$. The next search is performed in the interval $(l,u)$. This is repeated for a small number of search iterations, and the final value of $m$ that achieves validation accuracy at least $t$ is output as $\milr$.

\subsection{Experimental Setup}
\label{section:setup}

We primarily focus on constant width, fully-connected deep ReLU networks trained with SGD, that are initialized with the Kaiming normal initialization scheme. The batch size is set to 128 across all our experiments. We discuss reasons for making this choice of initialization in \autoref{section:sharpness}.

When using \autoref{alg:MILR}, we set threshold accuracy $t$ to be the accuracy that a linear classifier achieves on the given dataset, along with the number of training epochs $e = 10$. This ensures that networks trained at the maximal initial learning rate perform adequately while taking into account task difficulty. Namely, for MNIST and CIFAR-10 we use $t=0.925$ and $t=0.34$ respectively.
Upper and lower learning rate limits $u$ and $l$ are set heuristically; we use $l = 0.0$ for all our experiments, and find that $s = 5$ search iterations are sufficient in practice to calculate $\milr$. For computing the sharpness $\lambda_1$, we use PyHessian \citep{9378171}.

\section{Main Empirical Results}

\begin{figure*}[ht]
        \centering
        \includegraphics[width=\textwidth]{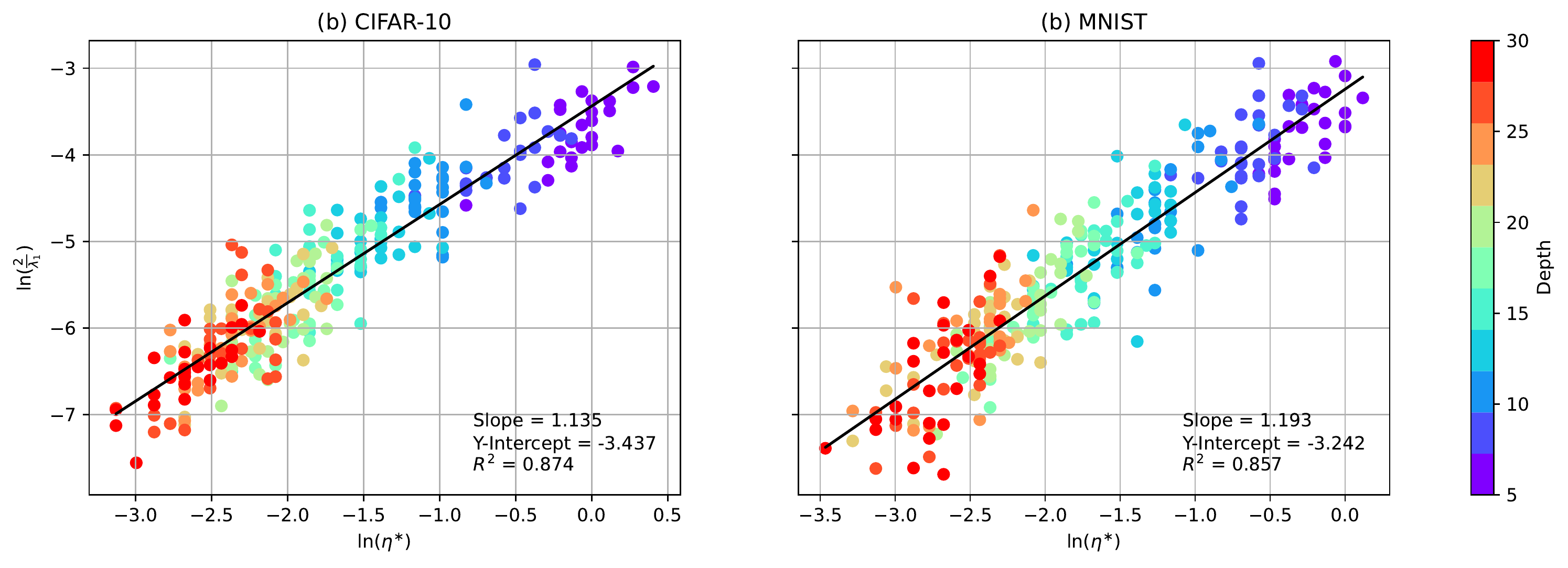}
        \caption{Plots of $2/\lambda_1$ against $\milr$ for 25 initializations per architecture, evaluated on CIFAR-10 and MNIST. We use architectures with $\text{width}  / \text{depth} = 16$, with the same depths as in \autoref{fig:const_width_depth_smallinputlr}.
        We find that $\milr > 2/\lambda_{1}$ and that $\ln(2/\lambda_{1}) = \beta\ln \milr + \gamma_{2}$ for $\beta \neq 1$, strongly contrasting with patterns observed later in training in e.g.~\citet{cohen2021gradient}.}
        \label{fig:sharpness_milr_comparison}
\end{figure*}

\subsection{Maximal Initial Learning Rate and Architecture}

We have described the maximal initial learning rate $\milr$ and explored its behavior under various training setups and threshold accuracies $t$. We now consider how $\milr$ depends on network architecture. % and the learning rate of the input layer.
From  \autoref{fig:const_width_depth_smallinputlr}, we see that if the ratio $\text{width} / \text{depth}$ is sufficiently large and the input layer is trained at a sufficiently small learning rate, then the expected value of $\milr$ is related to $(\text{depth}\times \text{width})$ through a power law: 

$$\E[\ln\milr] = -\alpha\ln(\text{depth}\times \text{width}) + \gamma_{1}$$
We also note that network architectures that deviate from the trend show much smaller $\milr$. In our experiments, these networks sometimes fail to beat the performance of a linear classifier, leading to no valid $\milr$ being found at all. For experimental results on Fashion-MNIST, we point the reader to \autoref{fig:fashion-mnist} (a) in the Appendix.

We use a small learning rate for the input layer weights because in virtually all principled initialization schemes the learning rate of a weight depends on the width of the previous layer, with larger widths corresponding to smaller learning rates (see e.g. Table 1 in \citet{pmlr-v139-yang21c}). 

Thus, while our experiments utilize networks that have constant \textit{hidden} layer widths, this suggests that the maximal initial learning rate for input layer weights may differ from the maximal learning rate appropriate for deeper layers, especially when the input dimension is large compared to the network width. In accordance with these expectations, we find that the behavior of $\milr$ is essentially unchanged if the input layer weights are frozen at initialization, while at smaller values of $(\text{depth}\times \text{width})$ the maximal initial learning rate $\milr$ deviates from the power-law predictions we otherwise observe (see \autoref{appendix:milr_uniform_lr}).

In the context of the above experiments and results, we pose the following questions for consideration in future work:
\begin{itemize}
    \item Through theoretical analysis of the Fisher information matrix, \citet{pmlr-v89-karakida19a} have suggested that the largest learning rate ensuring global convergence of SGD should scale linearly with $(\text{depth}\times \text{width})$:
    $$\E[\ln\milr] \propto -\alpha\ln(\text{depth}\times \text{width}), \quad \alpha = 1$$
    However, none of the setups considered here display this behavior. While there is no direct contradiction here, since global convergence is not guaranteed when training at $\milr$, how can one explain the discrepancy between these regimes? 
    
    \item Is there a method for initializing the input layer that allows it to train normally while preserving the strong trends observed in \autoref{fig:const_width_depth_smallinputlr}? It is possible that answering this question could shed light on the conflict between different methods for initializing the input layer.
    
    \item What is the full functional relationship between $(\text{depth}\times \text{width})$ and $\milr$ at moderate to large values of $\text{width} / \text{depth}$?
\end{itemize}

\subsection{Relationship to Expected Sharpness at Initialization}

\begin{figure*}[htbp!]
    \centering
    \includegraphics[width=\textwidth]{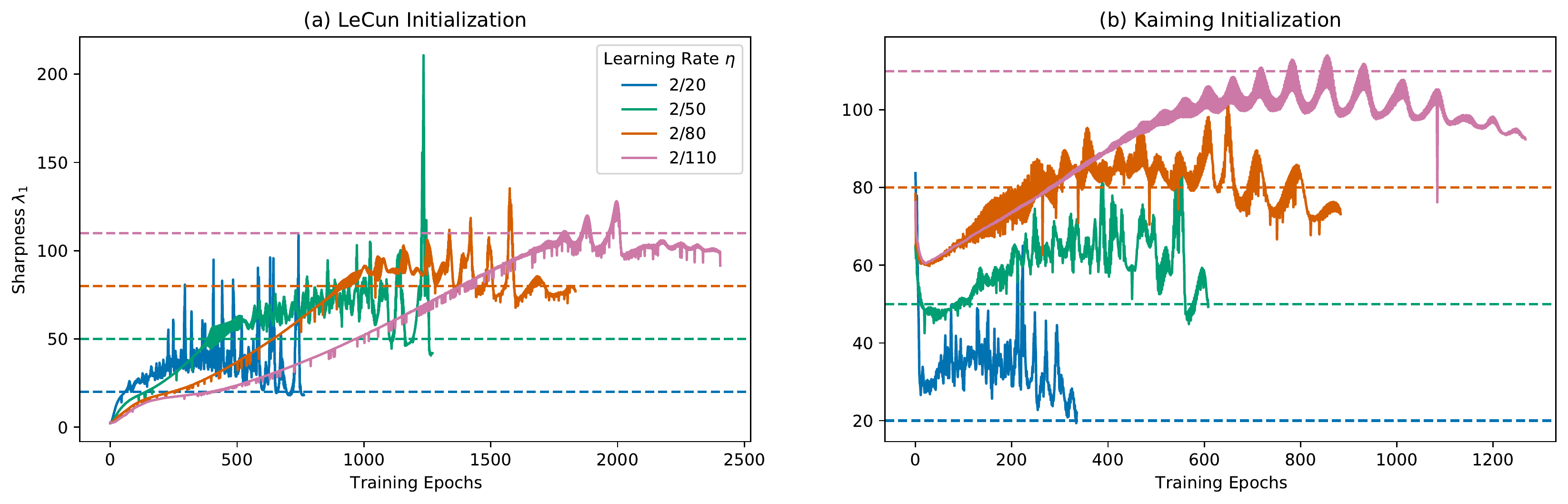}
    \caption{Sharpness $\lambda_{1}$ as a function of training epochs in full-batch gradient descent, showcasing ``progressive sharpening'' -- the tendency of $\lambda_{1}$ to continually rise until it reaches and hovers around the value of $2/\eta$ -- for (a) LeCun and (b) Kaiming initialization. We replicate the experimental setup of \citet{cohen2021gradient}, training the same network architecture until $99\%$ training accuracy is reached, for learning rates $\eta \in \{2/20,2/50,2/80,2/110\}$, on a 5k subset of CIFAR-10. Note that $\lambda_{1}$ at initialization is significantly larger for Kaiming-initialized networks than for LeCun-initialized networks (also see \autoref{fig:sharpness_stats} in the Appendix). The steep drop in sharpness at the beginning of training these networks may impact the value of $\milr$ since the maximal initial learning rate depends on the state of the network in early-stage training as well as at initialization itself.}
    \label{fig:eos_comparison}
\end{figure*}

The \emph{sharpness} $\lambda_{1}$ of a network is defined as the maximum eigenvalue of the training loss Hessian, and is often associated with the learning rate at which the network is trained. In particular, classical optimization tells us that the learning rate must be no larger than $2/\lambda_1$ to guarantee the convergence of SGD to the global minimum. However, this notion has lately been questioned in the context of deep neural networks \citep{cohen2021gradient, DBLP:journals/corr/abs-2003-02218}.

This motivates a need for a deeper understanding of sharpness and its connection to the learning rate and architecture in deep neural networks. To this end, we now explore how the expected sharpness of a network at initialization relates to architecture and maximal initial learning rate. In \autoref{fig:sharpness_const_width_depth}, we consider the value of $2/\lambda_1$ at initialization as a function of $(\text{depth}\times \text{width})^{-1}$, finding a power law relationship as with $\milr$. 

Note that the sharpness at initialization is the same regardless of whether the input layer is trained at the same or smaller learning rate as the rest of the network since no training is involved in computing sharpness at initialization.

Next, we directly compare the sharpness and the maximal initial learning rate at initialization (see \autoref{fig:sharpness_milr_comparison}, and \autoref{fig:fashion-mnist} (b) in the Appendix), using networks with $\text{width}  / \text{depth}=16$ and training the input layer at a small learning rate as in \autoref{fig:const_width_depth_smallinputlr}. While the work of \citet{cohen2021gradient} suggests that $\eta \sim 2/\lambda_{1}$ as networks converge to global optima, we find that at initialization the data closely fit a different power law, with $\ln(2/\lambda_{1}) \sim \beta \ln \milr$ for $\beta\ne 1$.
To study this comparison, we plot $2/\lambda_{1}$ as a function of $\milr$, using networks with $\text{width}  / \text{depth}=16$ and train the input layer at a small learning rate. This is done in order to compute $\milr$ that preserves the correlations observed in previous experiments.

Since judging linear fits on a log-log plot can be difficult, we provide another version of this figure in the Appendix in \autoref{fig:sharpness_vs_critlr_averaged}, averaging over initializations for each architecture. This suggests that the coefficient $\beta\neq 1$ estimated from Figure \ref{fig:sharpness_milr_comparison} is unlikely to be a product of noise.

It is also worth noting that for each point in the above plots, $\milr$ is clearly greater than $2/\lambda_{1}$. Recall that by definition, the computed $\milr$ ensures that a network initialization performs at least as well as a linear classifier on a given dataset, without diverging. This goes against the established wisdom of ``$\eta \leq 2/\lambda_{1}$'' for the convergence of SGD, hence supporting the recent lines of work that question this notion.

\subsection{Relationship to Edge of Stability} \label{section:sharpness}

In the previous experiments, we considered the sharpness at initialization, but since the definition of the maximal initial learning rate involves the ability to train a network past initialization, it makes sense that it could also be influenced by the value of the sharpness immediately following initialization. To gain further insight into this behavior, we revisit the concept of ``progressive sharpening'' identified in \citet{cohen2021gradient}. This term refers to the tendency when training at learning rate $\eta$ for the sharpness $\lambda_{1}$ to continually rise until it reaches and hovers around the value of $2/\eta$. In \autoref{fig:eos_comparison}, we replicate the experimental setup of \citet{cohen2021gradient}. In particular, we do so for both the LeCun initialization used in \citet{cohen2021gradient}, which initializes weights from a uniform distribution on $[-1/\text{fan-in},1/\text{fan-in}],$ and for Kaiming initialization, which initializes weights from a centered Gaussian with variance $2/\text{fan-in}$.

We use Kaiming initialization in this paper both because it is the more common initialization and also because in some sense it is the ``correct'' way to initialize deep ReLU networks. Namely, \citet{hanin2018start} show that in ReLU networks, Kaiming initialization prevents the mean size of the activations from growing exponentially large or small as a function of the depth, which occurs in LeCun initialization. \citet{hanin2018neural} shows a similar benefit to Kaiming initialization in reducing the problem of gradient explosion \citep{279181}. 

\autoref{fig:eos_comparison} illustrates that $\lambda_{1}$ at initialization scales quite differently for the two initialization schemes. For LeCun-initialized networks, $\lambda_{1}$ exponentially vanishes as depth increases (an effect more visible in \autoref{fig:sharpness_stats} in the Appendix, and described formally in \autoref{cor:lambda1-formal}), while it increases modestly with width and depth for Kaiming-initialized networks.

\begin{figure*}[!ht]
    \begin{minipage}[c]{0.5\textwidth}
        \includegraphics[width=0.9\textwidth]{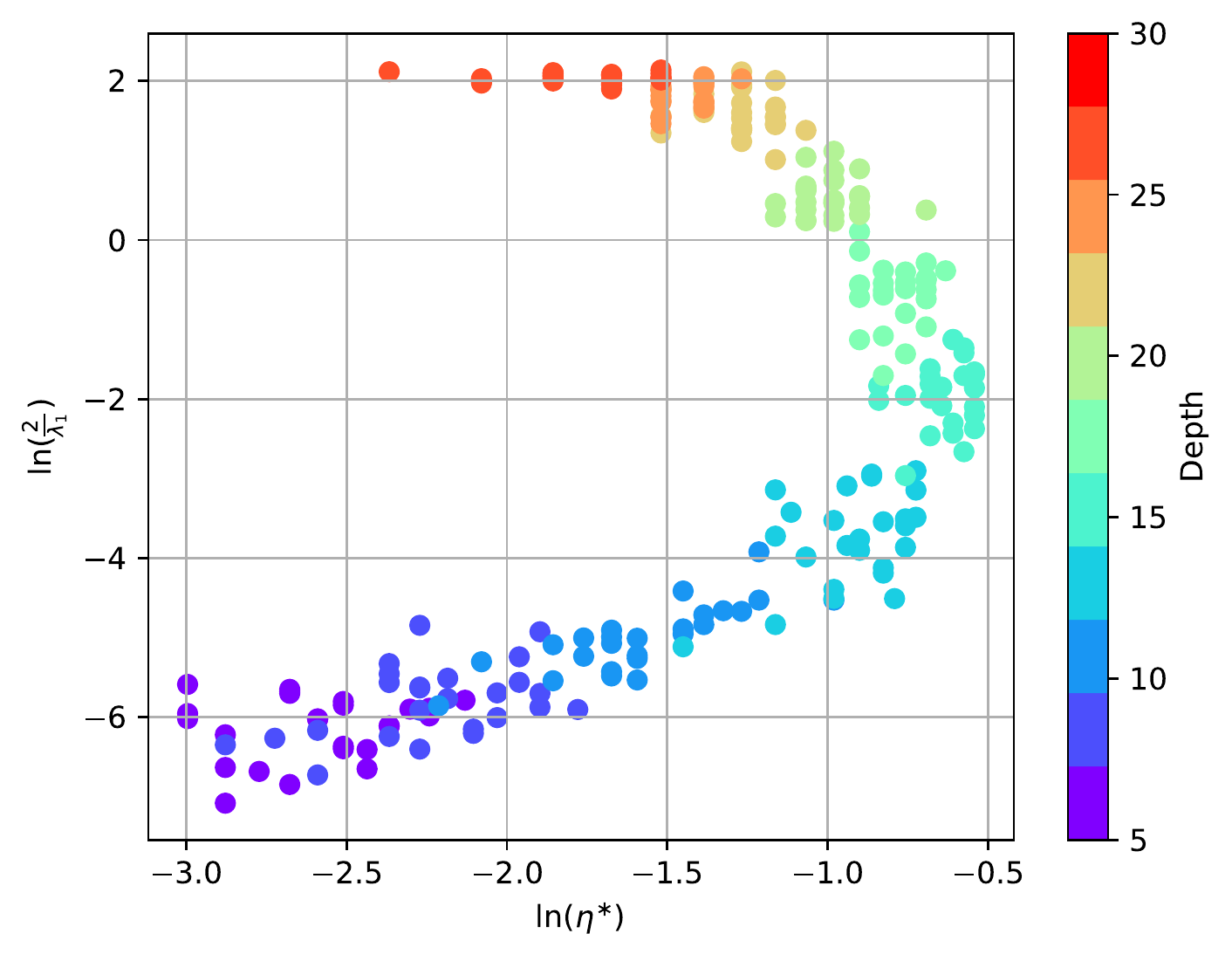}
    \end{minipage} \hfill
    \begin{minipage}[c]{0.5\textwidth}
        % \vspace{0.5cm}
        \caption{Plot of $2/\lambda_{1}$ against $\milr$ evaluated on CIFAR-10 for networks initialized with Neural Tangent Kernel (NTK) parametrization. Other experimental details are identical to those in \autoref{fig:sharpness_milr_comparison}. We find that in NTK-initialized networks, $2/\lambda_{1}$ and $\milr$ display a relationship very different from that observed in Kaiming-initialized networks. We also note key differences in empirical results between Kaiming and NTK-initialized networks: (a) $2/\lambda_{1}$ increases with depth (i.e.~$\lambda_{1}$ decreases with depth), and (b) the relationship displayed between $2/\lambda_{1}$ and $\milr$ is non-monotonic -- as $2/\lambda_{1}$ (and network depth) increases, $\milr$ first increases and then decreases.}
        \label{fig:ntp}
    \end{minipage}
\end{figure*}

We find that the ``edge of stability" phenomenon occurs for both initialization schemes, but takes slightly different forms. Namely, $\lambda_{1}$ at initialization is much larger for Kaiming-initialized networks and steeply drops off at the beginning of training, before rising (or dropping more slowly) until it hovers slightly above the $2/\eta$ line. This behavior is pertinent to our consideration of $\milr$, since the precipitous drop of $\lambda_{1}$ in very early training means that it is possible $\milr$ is able to take on larger values than it could if the sharpness remained at its initial value. This is thus a possible explanation of the behavior observed in \autoref{fig:sharpness_milr_comparison}.

The Neural Tangent Kernel (NTK) parametrization \citep{jacot2018neural,sohl2020infinite} is an initialization scheme popularly used to analyze networks in the infinite width limit. In the context of our experiments, it provides an alternate scaling of network weights with respect to layer width. In \autoref{fig:ntp} we plot $2/\lambda_{1}$ against $\milr$ for NTK-initialized networks and find that this relationship is highly different from that observed in previous experimental results.

We conclude this section by recalling \citet{gilmer2022a}, which claims that ``When the learning rate only slightly exceeds $2/\lambda_1$, optimization is unstable until the parameters move to a region with smaller $\lambda_{1}$.''
This claim is contradicted by our \autoref{fig:eos_comparison} (specifically the pink curve for Kaiming initialization and $\eta = 2/110$), from which we find that in fact the learning rate need not exceed $2/\lambda_1$ for the parameters to move to a region with smaller $\lambda_{1}$. To elaborate, note that the value of $\lambda_{1}$ (in the pink curve in \autoref{fig:eos_comparison} (b)) at the beginning of training is much smaller than $2/ \eta$ (where $\eta = 2/110$) at initialization. Even in this case, we observe that parameters move to a region with smaller $\lambda_{1}$, implying that the condition stated by \citet{gilmer2022a} is not necessary for this to occur.

\section{Estimates for the Largest Eigenvalue of the Loss Hessian at Initialization}\label{sec:thm}

In this section, we present our main theoretical result, Theorem \ref{thm:HS}, which computes the average squared Frobenius norm of the loss Hessian at initialization. Before stating it exactly, we present informally a simple corollary that gives bounds for the largest eigenvalue of the loss Hessian at initialization. 
\begin{corollary}[Informal]\label{cor:lambda1}
Consider a randomly initialized fully connected ReLU network of constant width, and denote by $\lambda_1$ the largest eigenvalue of the Hessian of an empirical MSE loss. We have the following upper bound on the largest eigenvalue:
\[
\Ee{\abs{\lambda_1}}=O( \text{depth}\times\text{width} )
\]
and the following lower bound on the sharpness:
\[
\Ee{\frac{2}{\text{sharpness}}}=\Ee{\frac{2}{\abs{\lambda_1}}}= \Omega\lrr{\frac{1}{\text{depth}\times\text{width} }},
\]
where the average in both estimates is over initialization.
\end{corollary}
\vspace{.1cm}

The preceding estimates show that depth times width, the key parameter which we found determines the maximal initial learning rate and the sharpness, naturally appears when computing the Frobenius norm of the loss Hessian (see also \autoref{thm:HS}) and hence can be used to obtain bounds on $\abs{\lambda_1}$ and sharpness at initialization.

\subsection{Formal Statements}
In order to state this Corollary and Theorem \ref{thm:HS} more precisely, we need some notation. We consider a ReLU network, which for an input $x\in \R^{n_0}$ outputs $z_{1}^{(L+1)}(x)\in \R$ via hidden layer pre-activations $z^{(\ell)}(x)\in \R^{n_\ell}$ as follows:
\[
z_{i}^{(\ell+1)}(x)=\begin{cases}
\sum_{j=1}^{n_\ell}W_{ij}^{(\ell)}\sigma\lrr{z_{j}^{(\ell)}(x)},&\quad \ell\geq 1\\
\sum_{j=1}^{n_\ell}W_{ij}^{(1)}x_{j},&\quad \ell=0
\end{cases}
\]
for $i = 1,\ldots,n_\ell$. Note that we have set the biases in the network to $0$. Moreover, we will assume that the weights are independent Gaussians $W_{ij}^{(\ell)}\sim \mathcal N\lrr{0,2/n_{\ell-1}}$.

Our goal is to understand the Hessian 
\[
H^{(L+1)}:= \big( \partial_{W_{ij}^{(\ell)}}\partial_{W_{i'j'}^{(\ell')}}\mathcal L\big),
\]
(here the indices summarize $1\leq \ell,\ell'\leq L+1,\, 1\leq i,i'\leq n_{\ell},\, 1\leq j,j'\leq n_{\ell-1}$) of the empirical MSE
\[
\mathcal L = \frac{1}{2k}\sum_{i=1}^k \lrr{z_1^{(L+1)}(x_i)-y_i}^2
\]
over $k$ input-output pairs $(x_i,y_i)$. Specifically, we compute the mean squared Frobenius norm of $H^{{(L+1)}}$ given by
\[
\Ee{\sum_{\ell,\ell'=1}^{L+1}\sum_{i,i'=1}^{n_{\ell}}\sum_{j,j'=1}^{n_{\ell-1}}\lrr{ \partial_{W_{ij}^{(\ell)}}\partial_{W_{i'j'}^{(\ell')}}\mathcal L }^2},
\]
where the average is over the Gaussian distribution of the weights. Our main result is
\begin{theorem}\label{thm:HS}
Fix $n_0\geq 1$ as well as $c,C>0$ and a network input $x$ satisfying $\norm{x}^2=n_0$. Then, there exists a constant $C_1$, depending only on $c,C$ with the following property. For any $L\geq 1$ there exists a constant $C_2$, depending only on $L,c,C$ such that if $cn \leq n_1,\ldots, n_L\leq Cn$, then
\[
 \abs{\Ee{\norm{H^{(L+1)}}_{F}^2} - C_1 n^2 L^2 }\leq C_2 n.
\]
\end{theorem}
\vspace{.1cm}
The preceding Theorem gives the following upper bound on the largest eigenvalue of $H^{(L+1)}$ and its reciprocal:
\begin{corollary}[Precise Statement of Corollary \ref{cor:lambda1}]\label{cor:lambda1-formal}
With the notation of Theorem \ref{thm:HS}, denote by $\lambda_1$ the largest eigenvalue of the Hessian of an empirical MSE loss. There exists a constant $K>0$ depending only on the constants $n_0, C,c$ from Theorem \ref{thm:HS} such that 
\[
\Ee{\abs{\lambda_1}}\leq KnL\lrr{1+O(n^{-1})}\quad \text{and}
\] 
\[
\Ee{2/\text{sharpness}}=\Ee{2/\abs{\lambda_1}}\geq  \frac{1}{KnL}\lrr{1+O(n^{-1})}.
\]
These results hold for Kaiming initialization. For LeCun initialization the same results hold, except $K$ must be replaced by $2^{-L/2}K$. 
\end{corollary}
\begin{proof}
Since the squared Frobenius norm of $H^{(L+1)}$ is the sum of squares of its eigenvalues, Theorem \ref{thm:HS} yields
\begin{align*}
\Ee{\abs{\lambda_1}}&\leq \sqrt{\Ee{\lambda_1^2}}\leq \sqrt{ \Ee{\norm{H^{(L+1)}}_F^2}}\\
&= C_1^{1/2} n L( 1+ O(n^{-1})).    
\end{align*}
Further, since $x\mapsto 1/x$ is convex on $(0,\infty)$, we have
\begin{align*}
\Ee{\frac{2}{\abs{\lambda_1}}}\geq \frac{2}{\Ee{\abs{\lambda_1}}}\geq \frac{2}{C_1^{1/2} n L}( 1+ O(n^{-1})).    
\end{align*}
Since depth $L$ ReLU networks are homogeneous of degree $L$ in their weights, The change from Kaiming to He initialization causes the network output (and hence its derivatives) to be re-scaled by a factor of $2^{-L/2}$. 
\end{proof}

\subsection{Proof Outline for Theorem \ref{thm:HS}}
Our strategy for estimating the Frobenius norm of $H^{(L+1)}$ is to proceed recursively in $L$. To explain the main idea (full details in the Appendix) we need some notation. First, we will use $\mu,\nu$ to denote generic variables indexing network weights. Next, for any $\ell=1,\ldots, L$ and any expressions $f_k(z)$ depending on $z$ and $\partial_\mu z$ we write
\[
\Y{\ell}{f_1,\ldots, f_k} :=\Ee{ \sum_{\mu\leq \ell} \frac{1}{n_\ell^k}\sum_{j_1,\ldots, j_k=1}^{n_\ell} \prod_{i=1}^k f_{i}(z_{j_i}^{(\ell)})}.
\]
Thus, for example,
\[
\Y{\ell}{z\partial_\mu z} = \Ee{ \sum_{\mu\leq \ell} \frac{1}{n_\ell}\sum_{j=1}^{n_\ell} z_j^{(\ell)}\partial_\mu z_{j}^{(\ell)}}.
\]
In both cases, $\mu\leq \ell$ denotes the collection of weights in layers $1,\ldots, \ell$. Similarly, if the functions $f_k(z)$ depend in addition on $\partial_\nu z$ and $\partial_{\mu\nu}z$ then we we will write
\[
    \Y{\ell}{f_1,\ldots, f_k} :=\Ee{ \sum_{\mu,\nu\leq \ell}  \frac{1}{n_\ell^k}\sum_{j_1,\ldots, j_k=1}^{n_\ell} \prod_{i=1}^k f_{i}(z_{j_i}^{(\ell)})}.
\]
Thus, for example, $\Y{\ell}{\lrr{\partial_\mu z}^2,z\partial_{\mu\nu}z}$ equals
\begin{align*}
\Ee{ \sum_{\mu\leq \ell}  \frac{1}{n_\ell^2}\sum_{j_1,j_2=1}^{n_\ell}\lrr{ \partial_\mu z_{j_1}^{(\ell)}}^2 z_{j_2}^{(\ell)}\partial_{\mu\nu}z_{j_2}^{(\ell)}}.    
\end{align*}
The key steps in proving Theorem \ref{thm:HS} are now as follows:
\begin{enumerate}
    \item Integrate out the weights in layer $L+1$ to rewrite $\Ee{\norm{H^{(L+1)}}_F^2}$ in terms of various $Y^{(L)}$'s. This is done in the Appendix in Lemma \ref{lem:H-rec} and Corollary \ref{cor:H-Y}. Since the Hessian involves second derivatives, the $Y^{(L)}$ that appear depend on various combinations of $z,\partial_\mu z,\partial_\nu z, \partial_{\mu\nu }z$.
    \item Obtain recursive expressions for the $Y^{(\ell+1)}$'s that depend only $z,\partial_\mu z$ in terms of the corresponding $Y^{(\ell)}$'s at the previous layer. This is done in Lemma \ref{lem:Ymu-recs}. Each such recursion is derived by considering two cases. First, the case where the parameter $\mu$ is a weight in layer $\ell+1$. This gives expressions no longer containing any derivatives that depend only on moments of the norm of the vector of pre-activations $z^{(\ell)}$ at layer $\ell$. Such moments are well-known. Second,  the case where the parameter $\mu$ is a weight in layers $1,\ldots, \ell$. By explicitly integrating out the weights in layer $\ell+1$, we obtain expressions involving various $Y^{(\ell)}$. 
    \item Solve the recursions for the $Y^{(\ell+1)}$'s that depend only on $z,\partial_\mu z$ to understand how they grow with depth and width. This is done in Corollary \ref{cor:Ymu-form}.
    \item Obtain consistent recursive expressions for the $Y^{(\ell+1)}$'s that depend only $z,\partial_\mu z, \partial_\nu z, \partial_{\mu\nu}z$ in terms of $Y^{(\ell)}$'s that depend only on the same expressions. This is done in Lemma \ref{lem:Ymunu-recs}. The strategy is the same as in deriving Lemma \ref{lem:Ymu-recs}, except one must consider three cases: $\mu,\nu$ are both weights in layer $\ell+1$, exactly one of $\mu,\nu$ is a weight in layer $\ell+1$, and neither of $\mu,\nu$ are weights in layer $\ell+1$. 
    \item Solve the recursions for the $Y^{(\ell+1)}$'s that depend only $z,\partial_\mu z, \partial_\nu z, \partial_{\mu\nu}z$ to understand how they grow with depth and width. This is done in Corollary \ref{cor:Ymunu-form}.
    \item Combine Corollaries \ref{cor:H-Y}, \ref{cor:Ymu-form}, and \ref{cor:Ymunu-form} to obtain estimates for the average of the squared Frobenius norm of $H^{(L+1)}$. 
\end{enumerate}

\section{Conclusion}
We have introduced the maximal initial learning rate along with a  simple algorithm to compute it. We empirically show that the maximal initial learning rate is closely related to the architecture and sharpness $\lambda_1$ at initialization in Kaiming-initialized fully-connected ReLU networks through:
\begin{align*}
    \E[\ln\milr] &= -\alpha\ln(\text{depth}\times \text{width}) + \gamma_{1},\\ %\quad \alpha \approx 0.75 \\
    \ln (2/\lambda_{1}) &= \beta \ln\milr + \gamma_{2},\quad \beta \neq 1
\end{align*}
as long as the network's $\text{width}/\text{depth}$ is sufficiently large and the input layer is trained at a relatively small learning rate. Further, we formally prove bounds for the sharpness in terms only of the value of  $(\text{depth} \times \text{width})$. 

To close, we emphasize several limitations and directions for future work. First, our experiments were performed only for constant width fully-connected ReLU networks trained by vanilla SGD with a fixed batch size. It would be therefore be interesting to further understand the architecture dependence of the maximal initial learning rate on: batch size, the presence of non-constant hidden layer widths, non-ReLU activations, non-fully connected layers, and the presence of normalization (e.g. BatchNorm). These factors can significantly impact network behavior early in training, which would, in turn, limit the direct application of our results in practical settings. For a rather preliminary investigation of maximal learning rates in ResNets see \autoref{appendix:resnet_performance}. 

Second, there is a rich vein of prior work concerning the dependence of learning rate and details of the optimization protocol. It would therefore be of interest to understand how the maximal initial learning rate $\milr$ varies with batch size \cite{goyal2017accurate,jastrzebski2018factors,hoffer2017train, smith2017don,smith2021origin} as well as  momentum coefficient, $\ell_2$ regularization strength, and data augmentation scheme. Similarly, it could be useful to study the relationship between architecture and $\milr$ when using adaptive optimizers such as Adam or Adagrad.

Further, both our experiments and theoretical analyses focused on optimization with a single fixed learning rate. In practice, however, learning rate protocols ranging from a simple learning rate drop after a fixed number of epochs to more intricate schemes such as warmup \cite{goyal2017accurate} or cosine schedules can improve network performance. Developing a theory of maximal learning rates that is valid throughout training could be of significant value.

Finally, we do not have a theoretical explanation that would predict the somewhat exotic power-law exponents $\alpha,\beta$ in the dependence of $\ln \milr$ on $(\text{depth}\times \text{width})$ and on $\ln(2/\lambda_1)$, and it would be interesting to understand their origin.

\section*{Acknowledgments}
D.R.~and G.I.~gratefully acknowledge support from the Canada CIFAR AI Chairs Program and NSERC Discovery Grants program. 

B.H.~gratefully acknowledges support from the NSF through DMS-2143754, DMS-1855684, and DMS-2133806 as well as support from an ONR MURI on Foundations of Deep Learning. 

The authors would like to thank Gintare Karolina Dziugaite and Devin Kwok for their helpful feedback during the early stages of this work. In addition, the authors acknowledge material support from NVIDIA and Intel in the form of computational resources and are grateful for technical support from the Mila IDT team in maintaining the Mila Compute Cluster.

\bibliography{milr_references}
\bibliographystyle{icml2023}

%%%%%%%%%%%%%%%%%%%%%%%%%%%%%%%%%%%%%%%%%%%%%%%%%%%%%%%%%%%%%%%%%%%%%%%%%%%%%%%
%%%%%%%%%%%%%%%%%%%%%%%%%%%%%%%%%%%%%%%%%%%%%%%%%%%%%%%%%%%%%%%%%%%%%%%%%%%%%%%
% APPENDIX
%%%%%%%%%%%%%%%%%%%%%%%%%%%%%%%%%%%%%%%%%%%%%%%%%%%%%%%%%%%%%%%%%%%%%%%%%%%%%%%
%%%%%%%%%%%%%%%%%%%%%%%%%%%%%%%%%%%%%%%%%%%%%%%%%%%%%%%%%%%%%%%%%%%%%%%%%%%%%%%
\newpage
\appendix
\onecolumn
\section{Performance of Fully-Connected Networks Trained with Maximal Initial Learning Rate} \label{appendix:performance}
\begin{figure}[h!]
    \centering
    \includegraphics[width=\textwidth]{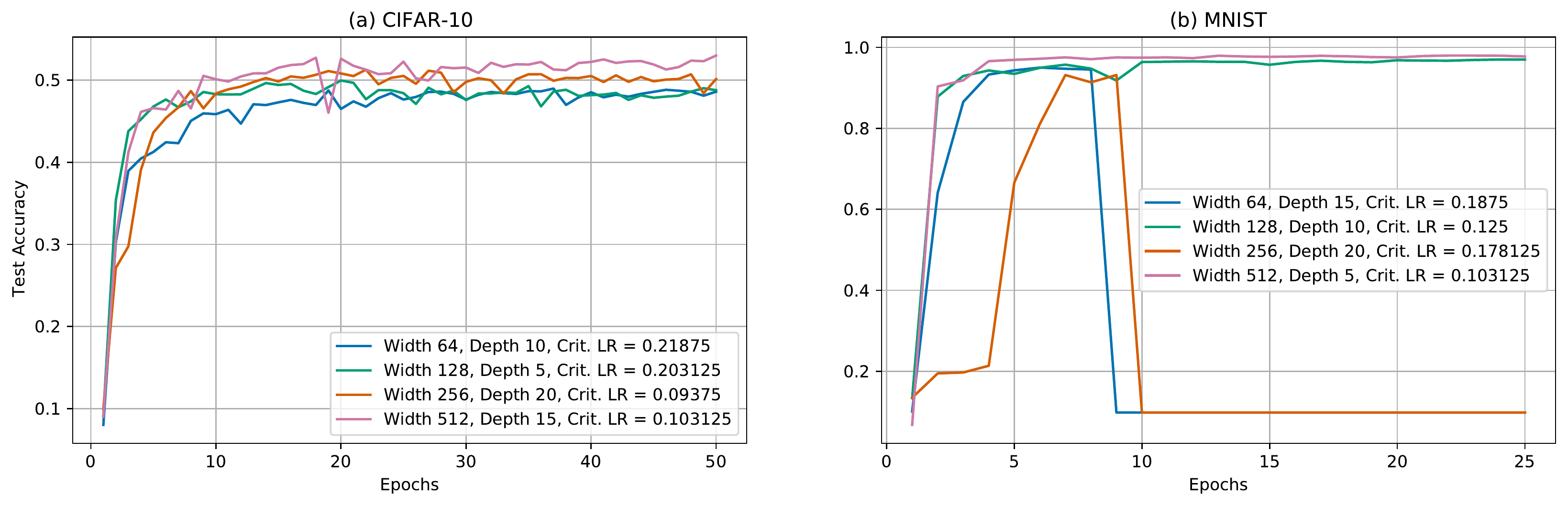}
    \caption{Performance of fully-connected networks when trained at the maximal initial learning rate $\milr$. 
    (a) Validation performance of initializations with different fully-connected network architectures when trained at learning rate $\eta = \milr$ computed by \autoref{alg:MILR} with $t = 0.34$ on CIFAR-10, over 50 training epochs. All initializations achieve $\approx 50\%$ validation accuracy, which is reasonable for fully-connected networks; (b) The same for MNIST with $t = 0.925$, over 25 training epochs. Note that training networks with $\milr$ guarantees that the network reaches the given threshold accuracy, but not long-term training stability. Refer to \S \ref{section:setup} for specifics of the experimental setup.}
    \label{fig:1}
\end{figure}

 Based on \autoref{fig:1}, we make the following observations about the maximal initial learning rate $\milr$:
\begin{enumerate}
    \item Networks train reasonably well when trained at $\milr$ -- the fully connected networks we consider achieve $\approx 50\%$ validation accuracy on CIFAR-10. Although the computed $\milr$ may not be small enough to achieve optimal performance when held constant, we believe they can serve well as large initial learning rates which can later be decayed for further improvement in performance.
    \item However, it must be noted that training at $\eta = \milr$ only guarantees that threshold performance will be achieved, and not long-term training stability. This is particularly easy to notice on ``easy'' datasets, such as MNIST. Fortunately, this can be easily overcome by simply employing early stopping, or a learning rate decay scheme.
\end{enumerate}

\section{Sharpness at Initialization for Different Initialization Schemes}
\label{appendix:sharpness_and_init}

\begin{figure*}[ht]
\centering
\includegraphics[width=\textwidth]{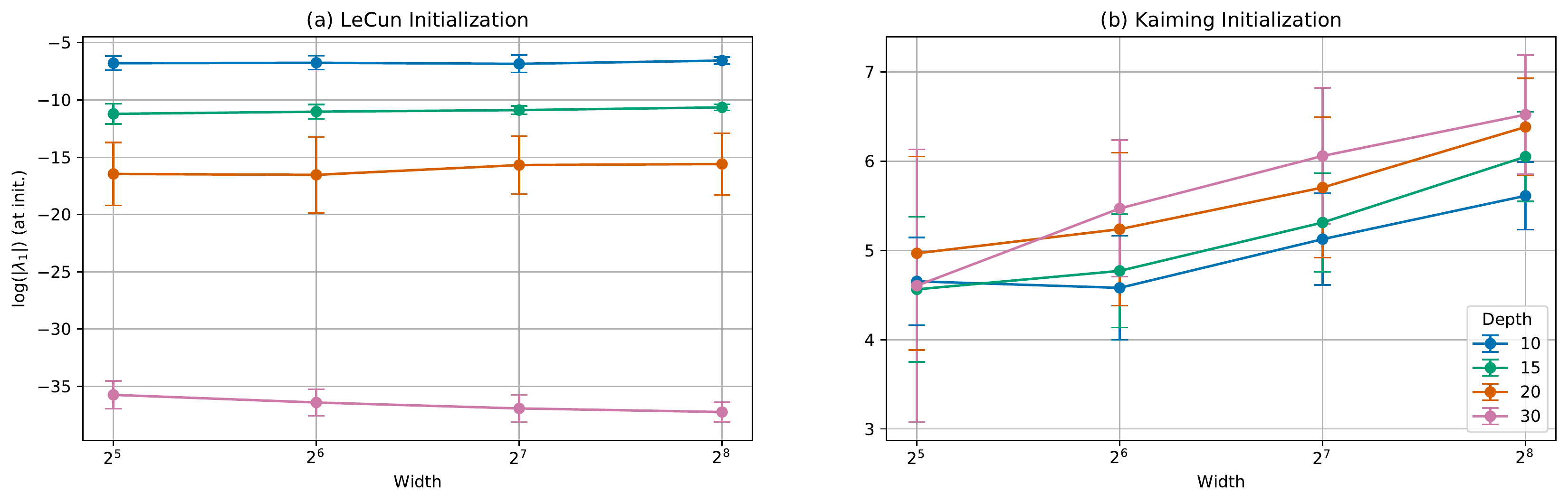}
\caption{Visualization of sharpness $\lambda_{1}$ as a function of width and depth at initialization for (a) LeCun \citep{lecun2012efficient} and (b) Kaiming \citep{He_2015_ICCV}  initialization schemes. We take absolute values and log scaling for $\lambda_{1}$ in (b) for the sake of clear representation - $\lambda_{1}$ values are extremely small in magnitude and can be positive or negative in sign. Width is given on the x-axis, and the different colors indicate different depths. There is a clear difference in how $\lambda_{1}$ scales for the considered initialization schemes -- while it becomes larger as width and depth increase for Kaiming-initialized networks, it becomes smaller with depth for LeCun-initialized networks.}
\label{fig:sharpness_stats}
\end{figure*}

\autoref{fig:sharpness_stats} compares $\lambda_{1}$ at initialization between the Kaiming and LeCun initialization schemes, as a function of changing fully-connected network architecture. Based on this figure, we make the following observations:

\begin{itemize}
    \item For Kaiming-initialized networks, $\lambda_{1}$ scales with both width and depth, and its variance primarily scales with the depth of the network architecture.
    \item LeCun-initialized networks show very little variation in $|\lambda_{1}|$ with a change in width, but considerably more with depth. Furthermore, $|\lambda_{1}|$ decreases exponentially as the depth gets larger, which is opposite to the trend we noticed for Kaiming initializations. It is also interesting to note that $\lambda_{1}$ may be positive or negative for LeCun-initialized networks, whereas Kaiming-initialized networks show only positive $\lambda_{1}$ values.
\end{itemize}

The above observations illustrate that the way in which $\lambda_{1}$ varies and scales with architecture is largely dependent on the initialization scheme employed. While we have the means to estimate the top eigenvalue(s) of the training Hessian, we do not yet \textit{exactly} understand how it is impacted by network architecture, data, and initialization. This is partly because the complete Hessian is difficult to compute and theoretically analyze in large-scale settings. We believe that a better understanding of this quantity could help us understand the role that sharpness plays in the optimization of neural networks.

\section{Maximal Initial Learning Rates with Standard Training Setup}
\label{appendix:milr_uniform_lr}

When all layers of the network are trained at the same learning rate $\milr$, the trend observed in \autoref{fig:const_width_depth_smallinputlr} breaks, and the relationship becomes non-linear at small $(\text{depth} \times \text{width})$ values, especially for networks with small $\text{width} / \text{depth}$. It is also worth noting that the linear relationship is preserved for much larger values of $(\text{depth} \times \text{width})$.

This raises a few questions: What exactly is the influence of the input layer on $\milr$? What is the correct way to initialize it so we see a linear trend between $\milr$ and $(\text{depth} \times \text{width})$? 

\begin{figure*}[htb!]
    \centering
    \includegraphics[width = \textwidth]{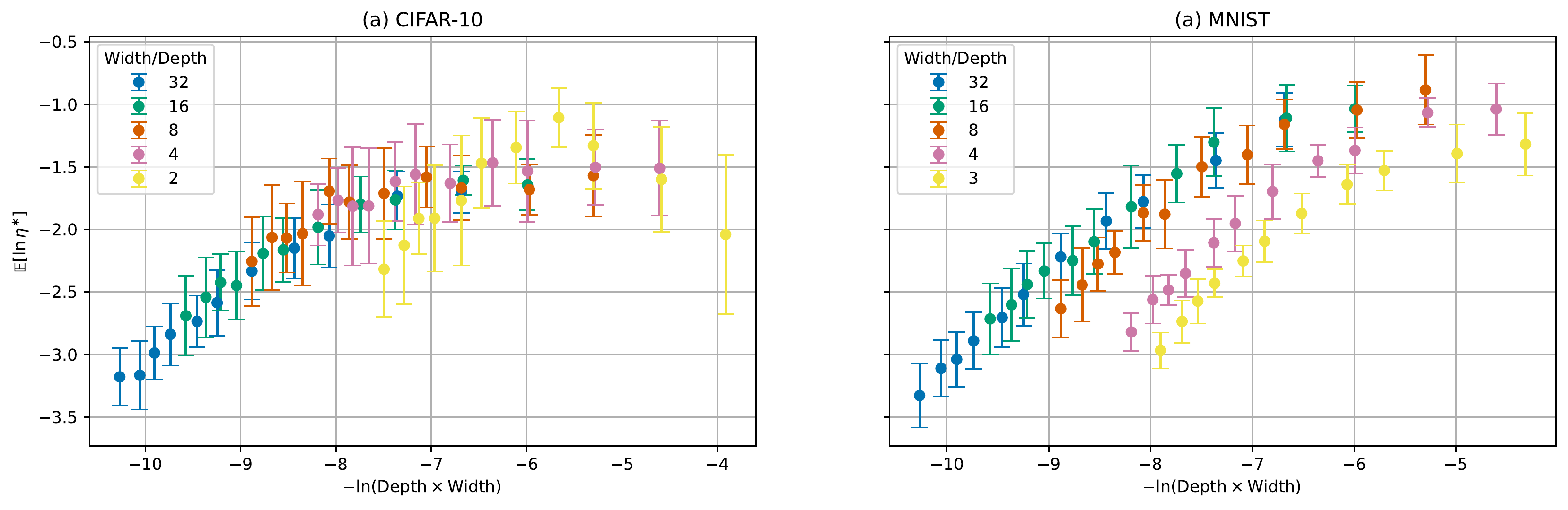}
    \caption{Relationship between the maximal initial learning rate $\milr$ and architecture for fully-connected networks with different $\text{width}  / \text{depth}$ values, trained on CIFAR-10 and MNIST. We use the same network architectures as in \autoref{fig:const_width_depth_smallinputlr}. We observe a consistent power relationship for networks with relatively large widths and small depths. However, this soon becomes non-linear for other, relatively deeper architectures.}
    \label{fig:const_width_depth_standard}
\end{figure*}

\section{Results on Fashion-MNIST}

\begin{figure}[H]
    \centering
    \includegraphics[width = \textwidth]{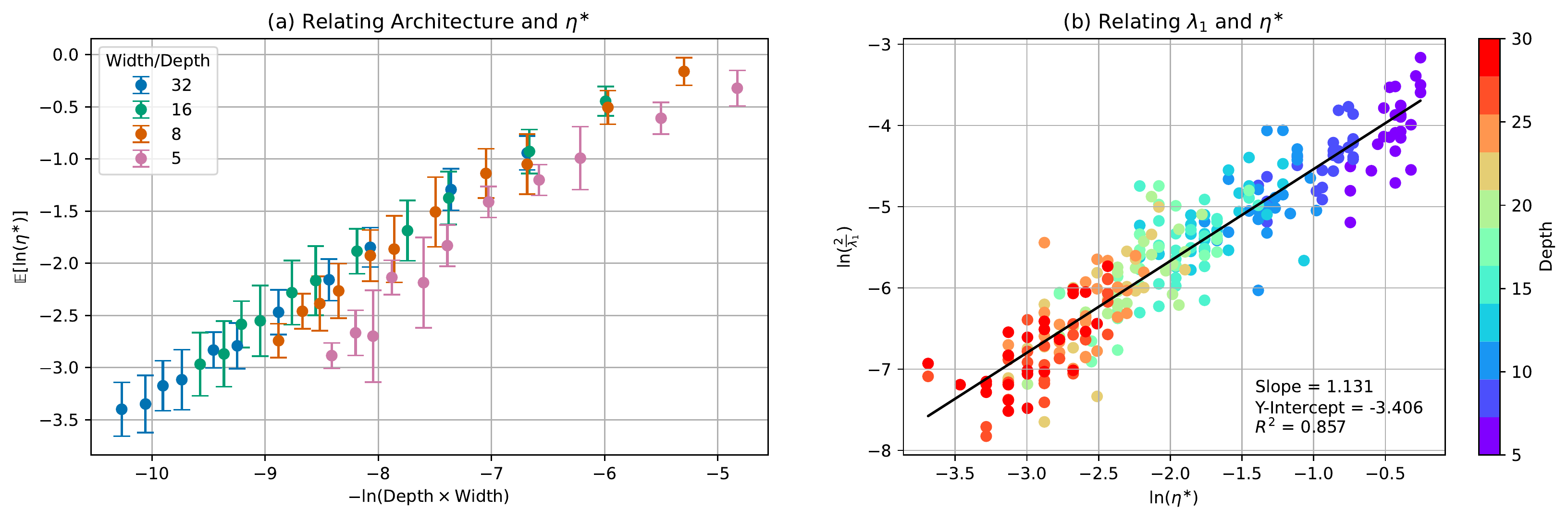}
    \caption{Experimental results for Fashion-MNIST. We obtain a threshold accuracy of 0.84 for Fashion-MNIST. The experimental setup remains identical to those in previous experiments, and the results further confirm the empirical and theoretical results obtained in this work.}
    \label{fig:fashion-mnist}
\end{figure}

\section{Results on Gaussian Data}

\begin{figure}[htb!]
    \centering
    \includegraphics[width=\textwidth]{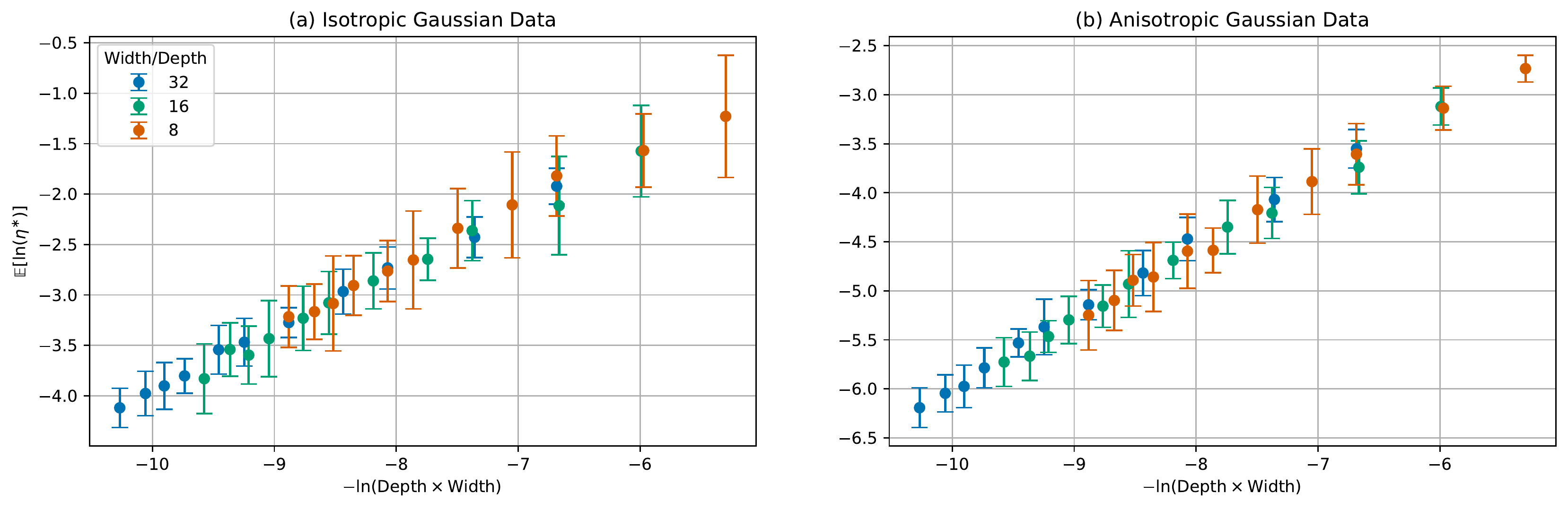}
    \subfloat[\empty]{
        \centering
        \hspace{1cm}
        \begin{tabular}{|c|c|c|}
                \hline
                \textbf{$\text{width} / \text{depth}$} & \textbf{Slope $\alpha$} & \textbf{$R^{2}$}  \\
                \hline \hline
                32 & 0.602 & 0.996 \\
                \hline
                16 & 0.599 & 0.991 \\
                \hline
                8 & 0.583 & 0.99 \\
                \hline
            \end{tabular}
        \hspace{3cm}
        \begin{tabular}{|c|c|c|}
                \hline
                \textbf{$\text{width} / \text{depth}$} & \textbf{Slope $\alpha$} & \textbf{$R^{2}$}  \\
                \hline \hline
                32 & 0.741 & 0.997 \\
                \hline
                16 & 0.714 & 0.995 \\
                \hline
                8 & 0.708 & 0.996 \\
                \hline
            \end{tabular}  
    }
    \caption{Relationship between maximal initial learning rate $\milr$ and architecture for (a) isotropic and (b) anisotropic Gaussian datasets. We use the same values for depth as in \autoref{fig:const_width_depth_smallinputlr}. Data is sampled from two multivariate normal distributions (i.e. binary classification). The training set and validation set respectively consist of 9k and 1k samples from each distribution, leading to a total of 20k samples (with 18k samples in the training set, and 2k in the validation set). Each sample is 100-dimensional, and the means are sampled from a standard normal distribution. For the anisotropic Gaussian dataset, we sample 100-dimensional covariance matrices from a standard normal distribution as well. For the isotropic and anisotropic Gaussian datasets, we obtain threshold accuracies of 1.0 and 0.81 respectively.}
    \label{fig:gaussian_data}
\end{figure}

We note that for the isotropic Gaussian data, the slope values are significantly smaller than those observed in other experiments. This emphasizes that even in the relatively simple case of fully connected ReLU networks, we do not have a theoretical explanation of the empirical scaling laws for the maximal initial learning rate as a function of architecture. Through these experiments, we hope to understand these simple situations before analyzing more complex cases.

\section{Comparison of Averaged Sharpness $\lambda_{1}$ and Maximal Initial Learning Rate {$\milr$}} \label{appendix: averaged_comparison}

\begin{figure}[H]
    \centering
    \includegraphics[width = \textwidth]{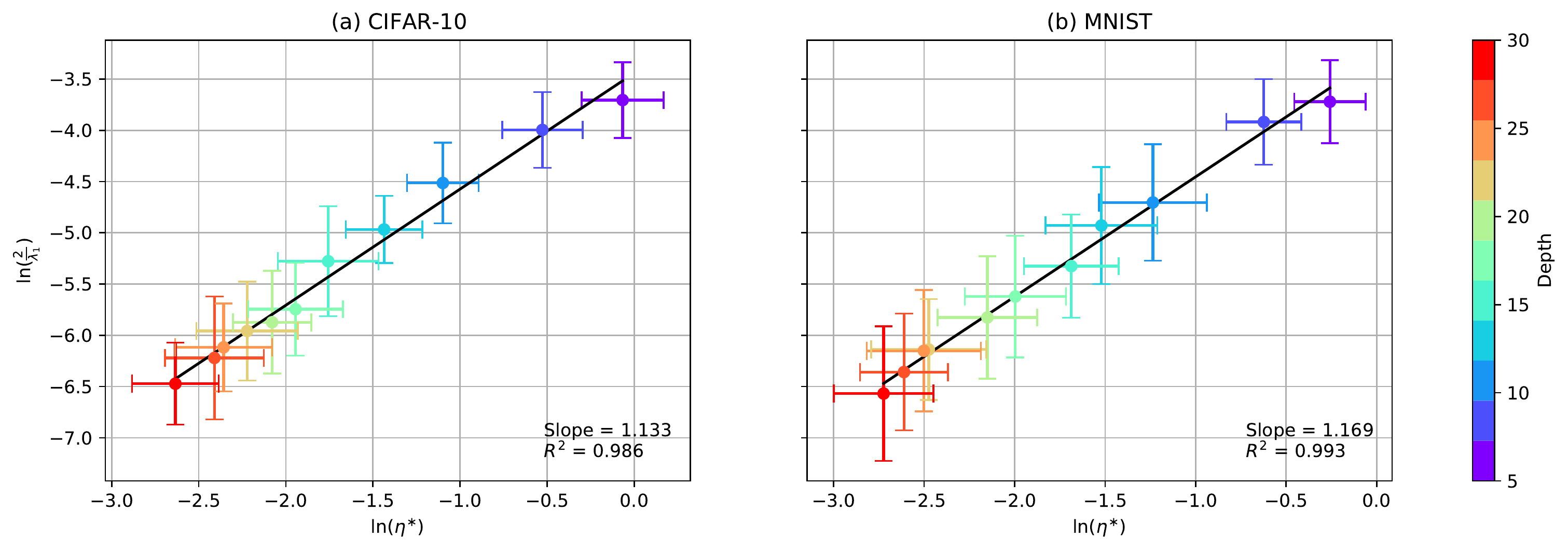}
    \caption{Correlation between $2/\lambda_1$ and $\milr$, averaged over 25 initializations per architecture. Each architecture has a sufficiently large $\text{width}  / \text{depth} = 16$, to preserve the established power relationship.}
    \label{fig:sharpness_vs_critlr_averaged}
\end{figure}

\section{Performance of ResNet-20 Networks Trained with Maximal Initial Learning Rate}
\label{appendix:resnet_performance}

\begin{figure}[htbp!]
    \centering
    \includegraphics[width=0.7\textwidth]{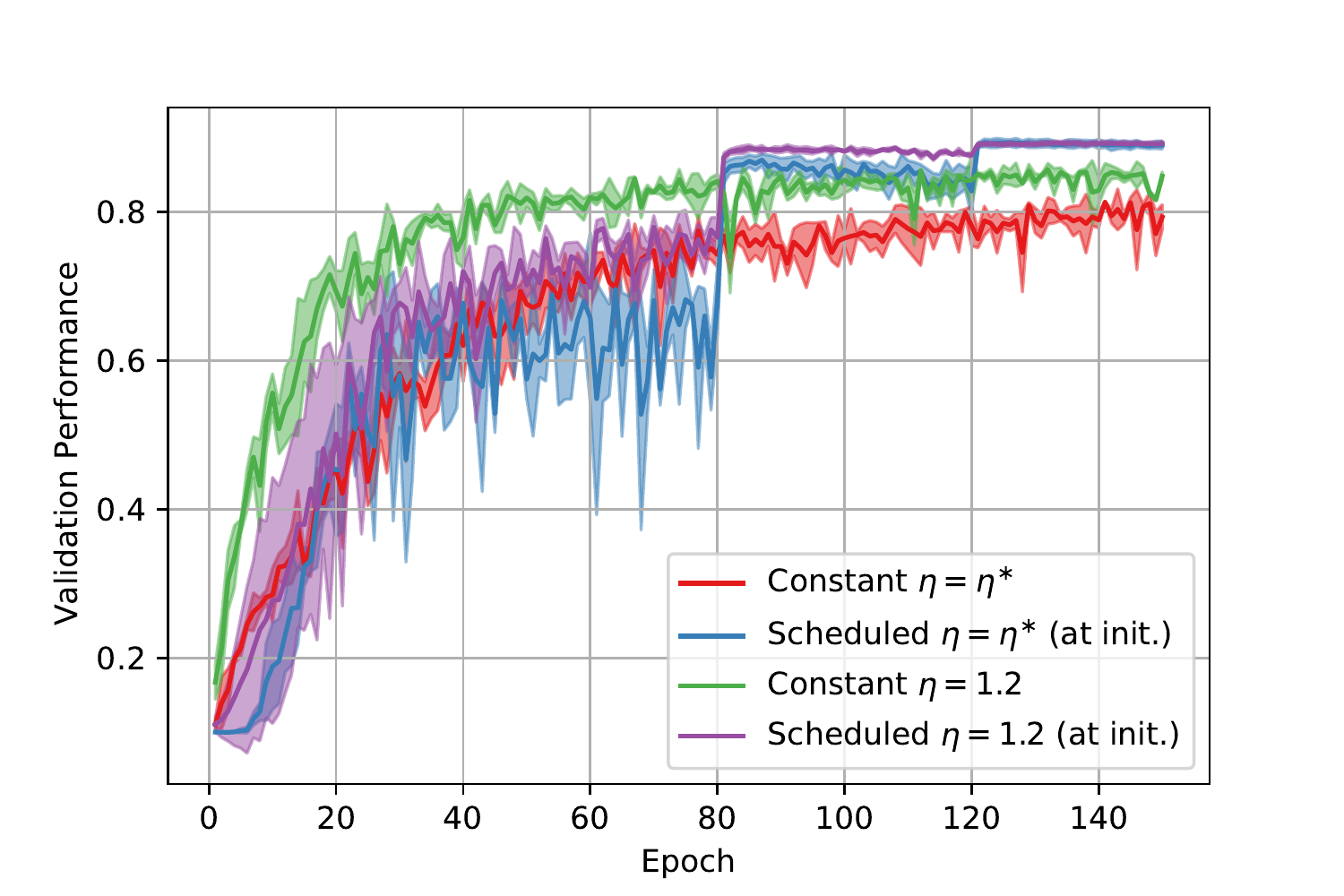}
    \caption{Performance of ResNet-20 \cite{He_2016_CVPR} networks with different learning rate setups. Each line in the figure is an average of 3 runs, along with error bars to indicate deviation in performance. A well-tuned, constant learning rate consistently beats $\milr$, but performance is competitive when using a scheduler is employed. Refer to \href{https://docs.mosaicml.com/en/v0.11.0/model_cards/cifar_resnet.html}{MosaicML's Model Card} for details of the learning rate scheduler setup.}
    \label{fig:milr_resnet_perf}
\end{figure}

\section{Proof of Theorem \ref{thm:HS}}\label{sec:pf}

\subsection{Setup and Preparatory Lemmas} Let us first recall the notation. We consider a ReLU network, which for an input $x\in \R^{n_0}$ computes $z_{1}^{(L+1)}\in \R$ via intermediate representations $z^{(\ell)}\in \R^{n_\ell}
$
\[
z_{i}^{(\ell+1)}=\begin{cases}
\sum_{j=1}^{n_\ell}W_{ij}^{(\ell)}\sigma\lrr{z_{j}^{(\ell)}},&\quad \ell\geq 1\\
\sum_{j=1}^{n_\ell}W_{ij}^{(1)}x_{j},&\quad \ell=0
\end{cases},\qquad i = 1,\ldots,n_\ell.
\]
Moreover, we will assume that the weights are independent Gaussians
\[
W_{ij}^{(\ell)}\sim \mathcal N\lrr{0,\frac{2}{n_{\ell-1}}}\quad \text{independent}.
\]
Instead of simply considering the loss Hessian as in the statement of Theorem \ref{thm:HS}, we will study a slightly more general effective Hessian
\[
H_{\text{eff}}:= \lrr{\widehat{\eta}_\mu \widehat{\eta}_\nu \partial_{\mu\nu}\left\{\frac{1}{2}\lrr{z_{1}^{(L+1)}-y}^2\right\}}_{\mu,\nu},
\]
where $\mu,\nu$ run over all network weights and for any weight $\mu= W_{ij}^{(\ell)}$ we write 
\[
\widehat{\eta}_{W_{ij}^{(\ell)}} =n_{\ell-1}^{-1/2}  \eta^{(\ell)}
\]
for the corresponding learning rates. We've introduced the rescaled learning rates $\eta^{(\ell)}$ for weights in layer $\ell$ for notational convenience in what follows. Our goal is to compute the mean of the Hilbert-Schmidt norm
\begin{equation}\label{eq:HS-def}
    \Ee{\norm{H_{\text{eff}}}_{HS}^2} = \Ee{\sum_{\mu,\nu\leq L+1}\lrr{\widehat{\eta}_\mu\widehat{\eta}_\nu \partial_\mu \partial_{\nu}\left\{\frac{1}{2}\lrr{z_{1}^{(L+1)}-y}^2\right\}}^2},
\end{equation}
where we remind the reader that for any $\ell$ the notation $\mu\leq \ell$ means the set of all weights in layers $1,\ldots, \ell.$ In order to effectively evaluate \eqref{eq:HS-def}, we need two preparatory results. The first is well-known and can be found in Theorem 3 of \citet{hanin2018neural} and Proposition 2 of \citet{hanin2018products}
\begin{lemma}\label{lem:ind}
The indicator random variables ${\bf 1}_{\set{z_{i}^{(\ell)} > 0}}$ are independent of any even function of the network weights and of each other. Their marginal distribution is Bernoulli $1/2$.
\end{lemma}
The second result we need is a simple corollary of Lemma \ref{lem:ind}. To state it, we need some notation. For any $\ell=1,\ldots, L$ and any expressions $f_k(z)$ depending on $z$ and $\partial_\mu z$ we write
\[
    \Y{\ell}{f_1,\ldots, f_k} :=\Ee{ \sum_{\mu\leq \ell} \lrr{\widehat{\eta}_\mu}^2 \frac{1}{n_\ell^k}\sum_{j_1,\ldots, j_k=1}^{n_\ell} f_{1}(z_{j_1}^{(\ell)})\cdots f_k(z_{j_k}^{(\ell)})}.
\]
Thus, for example
\[
\Y{\ell}{z\partial_\mu z} = \Ee{ \sum_{\mu\leq \ell} \lrr{\widehat{\eta}_\mu}^2 \frac{1}{n_\ell}\sum_{j=1}^{n_\ell} z_j^{(\ell)}\partial_\mu z_{j}^{(\ell)}}.
\]
Similarly, if the functions $f_k(z)$ depend in addition on $\partial_\nu z$ and $\partial_{\mu\nu}z$ then we we will write
\[
    \Y{\ell}{f_1,\ldots, f_k} :=\Ee{ \sum_{\mu,\nu\leq \ell} \lrr{\widehat{\eta}_\mu\widehat{\eta}_\nu}^2 \frac{1}{n_\ell^k}\sum_{j_1,\ldots, j_k=1}^{n_\ell} f_{1}(z_{j_1}^{(\ell)})\cdots f_k(z_{j_k}^{(\ell)})}.
\]
Thus, for example
\[
\Y{\ell}{\lrr{\partial_\mu z}^2,z\partial_{\mu\nu}z} = \Ee{ \sum_{\mu\leq \ell} \lrr{\widehat{\eta}_\mu\widehat{\eta}_\nu}^2 \frac{1}{n_\ell^2}\sum_{j_1,j_2=1}^{n_\ell}\lrr{ \partial_\mu z_{j_1}^{(\ell)}}^2 z_{j_2}^{(\ell)}\partial_{\mu\nu}z_{j_2}^{(\ell)}}.
\]

We will use repeatedly the following Corollary of Lemma \ref{lem:ind}:
\begin{corollary}\label{cor:sig-z}
Fix $k\geq 1$ and suppose that 
\[
f_j(z) = \sigma(z)^{a_j}\lrr{\partial_{\mu}\sigma(z)}^{b_j}\lrr{\partial_{\nu}\sigma(z)}^{c_j}\lrr{\partial_{\mu\nu}\sigma(z)}^{d_j},\qquad j=1,\ldots,k
\] 
with $a_j+b_j+c_j+d_j$ being even for every $j$. Write
\[
\widehat{f}_j(z) := z^{a_j}\lrr{\partial_{\mu}z}^{b_j}\lrr{\partial_{\nu}z}^{c_j}\lrr{\partial_{\mu\nu}z}^{d_j},\qquad j=1,\ldots, k.
\]
Then, 
\begin{equation}\label{eq:Yf-1}
\Y{\ell}{f_1} = \frac{1}{2}\Y{\ell}{\widehat{f}_1},\qquad \widehat{f}_1(z) := z^{a}\lrr{\partial_{\mu}z}^{b}.    
\end{equation}
Further, 
\[
\Y{\ell}{f_1,f_2} = \frac{1}{4}\left[\Y{\ell}{\widehat{f}_1, \widehat{f}_2}+\frac{1}{n_\ell}\Y{\ell}{\widehat{f}_1\cdot \widehat{f}_2}\right].
\]
\end{corollary}
\begin{proof}
When $k=1$, we have
\begin{align*}
    \Y{\ell}{f_1} &=\Ee{\sum_{\mu,\nu \leq \ell}\lrr{\widehat{\eta}_\mu \widehat{\eta}_\nu}^2 \frac{1}{n_\ell}\sum_{j=1}^{n_\ell} \lrr{\sigma(z_j^{(\ell)})}^{a_j}\lrr{\partial_{\mu}\sigma(z_j^{(\ell)})}^{b_j}\lrr{\partial_{\nu}\sigma(z_j^{(\ell)})}^{c_j}\lrr{\partial_{\mu\nu}\sigma(z_j^{(\ell)})}^{d_j}}\\
    &=\Ee{\sum_{\mu,\nu \leq \ell}\lrr{\widehat{\eta}_\mu \widehat{\eta}_\nu}^2 \frac{1}{n_\ell}\sum_{j=1}^{n_\ell} \lrr{z_j^{(\ell)}}^{a_j}\lrr{\partial_{\mu}z_j^{(\ell)}}^{b_j}\lrr{\partial_{\nu}z_j^{(\ell)}}^{c_j}\lrr{\partial_{\mu\nu}z_j^{(\ell)}}^{d_j}{\bf 1}_{\set{z_j^{(\ell)}\geq 0}}}\\
    &=\Ee{\Ee{\sum_{\mu,\nu \leq \ell}\lrr{\widehat{\eta}_\mu \widehat{\eta}_\nu}^2 \frac{1}{n_\ell}\sum_{j=1}^{n_\ell} \lrr{z_j^{(\ell)}}^{a_j}\lrr{\partial_{\mu}z_j^{(\ell)}}^{b_j}\lrr{\partial_{\nu}z_j^{(\ell)}}^{c_j}\lrr{\partial_{\mu\nu}z_j^{(\ell)}}^{d_j}{\bf 1}_{\set{z_j^{(\ell)}\geq 0}}~\bigg|~z^{(\ell-1)}}}.
\end{align*}
In the inner conditional expectation, the term $\lrr{z_j^{(\ell)}}^{a_j}\lrr{\partial_{\mu}z_j^{(\ell)}}^{b_j}\lrr{\partial_{\nu}z_j^{(\ell)}}^{c_j}\lrr{\partial_{\mu\nu}z_j^{(\ell)}}^{d_j}{\bf 1}_{\set{z_j^{(\ell)}\geq 0}}$ is an even function of the weights in layer $\ell$. Hence, by Lemma \ref{lem:ind}, it is independent of the indicator function. This yields  \eqref{eq:Yf-1}. Similarly, we have
\begin{align*}
    \Y{\ell}{f_1,f_2} &=\Ee{\sum_{\mu,\nu \leq \ell}\lrr{\widehat{\eta}_\mu \widehat{\eta}_\nu}^2 \frac{1}{n_\ell^2}\sum_{k_1,k_2=1}^{n_\ell} \prod_{j=1}^2 \lrr{\sigma(z_{k_j}^{(\ell)})}^{a_j}\lrr{\partial_{\mu}\sigma(z_{k_j}^{(\ell)})}^{b_j}\lrr{\partial_{\nu}\sigma(z_{k_j}^{(\ell)})}^{c_j}\lrr{\partial_{\mu\nu}\sigma(z_{k_j}^{(\ell)})}^{d_j}}\\
    &=\Ee{\sum_{\mu,\nu \leq \ell}\lrr{\widehat{\eta}_\mu \widehat{\eta}_\nu}^2 \frac{1}{n_\ell^2}\sum_{k_1,k_2=1}^{n_\ell} \prod_{j=1}^2 \lrr{z_{k_j}^{(\ell)}}^{a_j}\lrr{\partial_{\mu}z_{k_j}^{(\ell)}}^{b_j}\lrr{\partial_{\nu}z_{k_j}^{(\ell)}}^{c_j}\lrr{\partial_{\mu\nu}z_{k_j}^{(\ell)}}^{d_j} {\bf 1}_{\set{z_{k_1}^{(\ell)}\geq 0}}{\bf 1}_{\set{z_{k_2}^{(\ell)}\geq 0}}}\\
    &= \frac{1}{n_\ell} \Ee{\sum_{\mu,\nu \leq \ell}\lrr{\widehat{\eta}_\mu \widehat{\eta}_\nu}^2 \lrr{z_{1}^{(\ell)}}^{a_1+a_2}\lrr{\partial_{\mu}z_{1}^{(\ell)}}^{b_1+b_2}\lrr{\partial_{\nu}z_{1}^{(\ell)}}^{c_1+c_2}\lrr{\partial_{\mu\nu}z_{1}^{(\ell)}}^{d_1+d_2} {\bf 1}_{\set{z_{j_1}^{(\ell)}\geq 0}}}\\
    &+\lrr{1-\frac{1}{n_\ell}} \Ee{\sum_{\mu,\nu \leq \ell}\lrr{\widehat{\eta}_\mu \widehat{\eta}_\nu}^2 \prod_{j=1}^2 \lrr{z_{j}^{(\ell)}}^{a_j}\lrr{\partial_{\mu}z_{j}^{(\ell)}}^{b_j}\lrr{\partial_{\nu}z_{j}^{(\ell)}}^{c_j}\lrr{\partial_{\mu\nu}z_{j}^{(\ell)}}^{d_j} {\bf 1}_{\set{z_{j}^{(\ell)}\geq 0}}},
\end{align*}
where the last equality follows by symmetry. Again conditioning on $z^{(\ell-1)}$ we thus find
\begin{align*}
    \Y{\ell}{f_1,f_2} 
    &= \frac{1}{2n_\ell} \Ee{\sum_{\mu,\nu \leq \ell}\lrr{\widehat{\eta}_\mu \widehat{\eta}_\nu}^2 \lrr{z_{1}^{(\ell)}}^{a_1+a_2}\lrr{\partial_{\mu}z_{1}^{(\ell)}}^{b_1+b_2}\lrr{\partial_{\nu}z_{1}^{(\ell)}}^{c_1+c_2}\lrr{\partial_{\mu\nu}z_{1}^{(\ell)}}^{d_1+d_2}}\\
    &+\frac{1}{4}\lrr{1-\frac{1}{n_\ell}} \Ee{\sum_{\mu,\nu \leq \ell}\lrr{\widehat{\eta}_\mu \widehat{\eta}_\nu}^2 \prod_{j=1}^2 \lrr{z_{j}^{(\ell)}}^{a_j}\lrr{\partial_{\mu}z_{j}^{(\ell)}}^{b_j}\lrr{\partial_{\nu}z_{j}^{(\ell)}}^{c_j}\lrr{\partial_{\mu\nu}z_{j}^{(\ell)}}^{d_j} }.
\end{align*}
Running the above symmetry argument in reverse yields
\[
 \Y{\ell}{f_1,f_2}  =\frac{1}{4}\left[ \Y{\ell}{\widehat{f}_1,\widehat{f}_2} +\frac{1}{n_\ell}\Y{\ell}{\widehat{f}_1\cdot \widehat{f}_2} \right],
 \]
 as claimed. 
\end{proof}
In what follows we will use Lemma \ref{lem:ind} and Corollary \ref{cor:sig-z} without mention.

\subsection{Reducing $\Ee{\norm{H_{eff}}_{F}^2}$ to $Y^{(\ell)}$'s}
To make progress on evaluating the expression \eqref{eq:HS-def}, let us first write
\[
\partial_\mu \partial_{\nu}\left\{\frac{1}{2}\lrr{z_{1}^{(L+1)}-y}^2\right\} = \partial_{\mu} \lrr{\partial_\nu z_{1}^{(L+1)}\lrr{z_{1}^{(L+1)}-y}} = \partial_{\mu\nu} z_{1}^{(L+1)}\lrr{z_{1}^{(L+1)}-y} + \partial_\mu z_{1}^{(L+1)} \partial_\nu z_{1}^{(L+1)}.
\]
Hence, 
\begin{align*}
    \lrr{\partial_\mu \partial_{\nu}\left\{\frac{1}{2}\lrr{z_{1}^{(L+1)}-y}^2\right\} }^2&= \lrr{\partial_{\mu\nu} z_{1}^{(L+1)}\lrr{z_{1}^{(L+1)}-y} }^2 + 2\partial_{\mu\nu} z_{1}^{(L+1)}\partial_\mu z_{1}^{(L+1)} \partial_\nu z_{1}^{(L+1)}\lrr{z_{1}^{(L+1)}-y}\\
    &+ \lrr{\partial_\mu z_{1}^{(L+1)} \partial_\nu z_{1}^{(L+1)}}^2.
\end{align*}
Using that $y$ has mean $0$ and variance $1$ as well as the fact that any term with an odd number of $z_{1}^{(L+1)}$'s has zero mean shows that
\begin{align}
\label{eq:d2L}     \Ee{\lrr{\partial_\mu \partial_{\nu}\left\{\frac{1}{2}\lrr{z_{1}^{(L+1)}-y}^2\right\} }^2}&= \Ee{\lrr{\partial_{\mu\nu} z_{1}^{(L+1)}z_{1}^{(L+1)}}^2}+ \Ee{\lrr{\partial_{\mu\nu} z_{1}^{(L+1)}}^2}\\
\notag     &+2\Ee{\partial_{\mu\nu} z_{1}^{(L+1)}\partial_\mu z_{1}^{(L+1)} \partial_\nu z_{1}^{(L+1)}z_{1}^{(L+1)}}+\Ee{\lrr{\partial_\mu z_{1}^{(L+1)} \partial_\nu z_{1}^{(L+1)}}^2}.
\end{align}
Our goal is to evaluate the sums of such terms over $\mu,\nu$ recursively in $L$. We do this by first integrating out the weights in the last layer to reduce computing the expected squared Hilbert-Schmidt norm of the loss hessian to various $Y$'s.

\begin{lemma}\label{lem:H-rec}
    We have
    \begin{align}
       \notag & \sum_{\mu,\nu\leq L+1} \Ee{\lrr{\widehat{\eta}_\mu\widehat{\eta}_\nu}^2\lrr{\partial_{\mu\nu} z_{1}^{(L+1)}z_{1}^{(L+1)}}^2}\\
       &\notag \qquad =\lrr{\eta^{(L+1)}}^2\left[\Y{L}{(\partial_\mu z)^2,z^2}+\frac{1}{n_L}\Y{L}{(z\partial_\mu z)^2}\right]+\Y{L}{(\partial_{\mu\nu}z)^2,z^2}+\frac{1}{n_L} \Y{L}{(z\partial_{\mu\nu}z)^2}\\
 \label{eq:rec1}  &\qquad +2\Y{L}{z\partial_{\mu\nu }z,z\partial_{\mu\nu }z}+\frac{2}{n_L}\Y{L}{(z\partial_{\mu\nu}z)^2}\\
      \notag  &\sum_{\mu,\nu\leq L+1}  \Ee{\lrr{\widehat{\eta}_\mu\widehat{\eta}_\nu}^2\lrr{\partial_{\mu\nu} z_{1}^{(L+1)}}^2}\\
      \label{eq:rec2}  &\qquad = \lrr{\eta^{(L+1)}}^2 \Y{L}{\lrr{\partial_\mu z}^2}
      +\Y{L}{\lrr{\partial_{\mu\nu} z}^2}\\
    \notag    & \sum_{\mu,\nu\leq L+1} \Ee{\lrr{\widehat{\eta}_\mu\widehat{\eta}_\nu}^2\partial_{\mu\nu} z_{1}^{(L+1)}\partial_\mu z_{1}^{(L+1)}\partial_\nu z_{1}^{(L+1)}z_{1}^{(L+1)}}\\
     \notag &\qquad =\lrr{\eta^{(L+1)}}^2\lrr{\Y{L}{z\partial_\mu z,z\partial_\mu z}+ \frac{1}{n_L}\Y{L}{(z\partial_\mu z)^2}}\\
     \label{eq:rec3}  &\qquad +2\Y{L}{\partial_{\mu \nu}z \partial_\mu z, z\partial_\nu z }+\Y{L}{z\partial_{\mu\nu}z, \partial_\mu z\partial_\nu z}+\frac{3}{n_L}\Y{L}{z\partial_{\mu \nu}z \partial_\mu z \partial_\nu z}\\
    \notag   &  \Ee{\lrr{\widehat{\eta}_\mu\widehat{\eta}_\nu}^2\lrr{\partial_{\mu} z_{1}^{(L+1)}\partial_\nu z_{1}^{(L+1)}}^2}\\
    &\notag \qquad =\frac{1}{2}\lrr{\eta^{(L+1)}}^4 \lrr{\Ee{\lrr{\frac{1}{n_L}\norm{z^{(L)}}_2^2}^2}+\frac{1}{n_L}\Ee{\frac{1}{n_L}\norm{z^{(L)}}_4^4}}\\
    &\notag \qquad + \lrr{\eta^{(L+1)}}^2\left[\Y{L}{z^2,\lrr{ \partial_\mu z}^2}+\frac{1}{n_L}\Y{L}{\lrr{z\partial_\mu z}^2}\right]\\
    &\label{eq:rec4} \qquad +2\Y{L}{\partial_\mu z \partial_\nu z,\partial_\mu z \partial_\nu z}+\Y{L}{\lrr{\partial_\mu z}^2,\lrr{\partial_\nu z}^2}+\frac{3}{n_L}\Y{L}{\lrr{\partial_\mu z \partial_\nu z}^2}
    \end{align}
\end{lemma}

\begin{proof}
We begin with deriving (\eqref{eq:rec1}). We have
\begin{align*}
    \sum_{\mu,\nu\leq L+1} \Ee{\lrr{\widehat{\eta}_\mu\widehat{\eta}_\nu}^2\lrr{\partial_{\mu\nu} z_{1}^{(L+1)}z_{1}^{(L+1)}}^2}&=
    \sum_{\mu,\nu \in L+1} \Ee{\lrr{\widehat{\eta}_\mu\widehat{\eta}_\nu}^2\lrr{\partial_{\mu\nu} z_{1}^{(L+1)}z_{1}^{(L+1)}}^2}\\
    &+\sum_{\mu\leq L,\, \nu \in L+1} \Ee{\lrr{\widehat{\eta}_\mu\widehat{\eta}_\nu}^2\lrr{\partial_{\mu\nu} z_{1}^{(L+1)}z_{1}^{(L+1)}}^2}\\
    &+\sum_{\mu\in L+1,\, \nu \leq L+1} \Ee{\lrr{\widehat{\eta}_\mu\widehat{\eta}_\nu}^2\lrr{\partial_{\mu\nu} z_{1}^{(L+1)}z_{1}^{(L+1)}}^2}\\
    &+\sum_{\mu,\nu\leq L} \Ee{\lrr{\widehat{\eta}_\mu\widehat{\eta}_\nu}^2\lrr{\partial_{\mu\nu} z_{1}^{(L+1)}z_{1}^{(L+1)}}^2}.
\end{align*}
Note first that if $\mu,\nu$ are both weights in layer $L+1$, then $\partial_{\mu\nu}z_{1}^{(L+1)}=0$. Thus, the first sum vanishes. Next, the second and third sums are equal. To evaluate them we proceed as follows:
\begin{align*}
    \sum_{\mu\leq L,\, \nu \in L+1} \Ee{\lrr{\widehat{\eta}_\mu\widehat{\eta}_\nu}^2\lrr{\partial_{\mu\nu} z_{1}^{(L+1)}z_{1}^{(L+1)}}^2}&=\sum_{\mu\leq L}\lrr{\widehat{\eta}_\mu}^2 \Ee{\frac{\lrr{\eta^{(L+1)}}^2}{n_L}\sum_{j_1=1}^{n_L}\lrr{\partial_\mu \sigma(z_{j_1}^{(L)}) z_{1}^{(L+1)}}^2}\\
    &=\sum_{\mu\leq L}\lrr{\widehat{\eta}_\mu}^2 \Ee{\frac{\lrr{\eta^{(L+1)}}^2}{n_L}\sum_{j_1=1}^{n_L}\lrr{\partial_\mu \sigma(z_{j_1}^{(L)}) \sum_{j_2=1}^{n_L}W_{1j_2}^{(L+1)}\sigma(z_{j_2}^{(L)})}^2}\\
    &=\sum_{\mu\leq L}\lrr{\widehat{\eta}_\mu}^2 \Ee{\frac{2\lrr{\eta^{(L+1)}}^2}{n_L^2}\sum_{j_1,j_2=1}^{n_L}\lrr{\partial_\mu \sigma(z_{j_1}^{(L)})\sigma(z_{j_2}^{(L)})}^2}\\
    &=\sum_{\mu\leq L}\lrr{\widehat{\eta}_\mu}^2 2\lrr{\eta^{(L+1)}}^2\Ee{\frac{1}{n_{L}}\lrr{\partial_\mu \sigma(z_{1}^{(L)})\sigma(z_{1}^{(L)})}^2}\\
    &+\sum_{\mu\leq L}\lrr{\widehat{\eta}_\mu}^2 2\lrr{\eta^{(L+1)}}^2\Ee{\lrr{1-\frac{1}{n_{L}}}\lrr{\partial_\mu \sigma(z_{1}^{(L)})\sigma(z_{2}^{(L)})}^2}\\
    &=\sum_{\mu\leq L}\lrr{\widehat{\eta}_\mu}^2 \lrr{\eta^{(L+1)}}^2\Ee{\frac{1}{n_{L}}\lrr{\partial_\mu z_{1}^{(L)} z_{1}^{(L)}}^2}\\
    &+\frac{1}{2}\sum_{\mu\leq L}\lrr{\widehat{\eta}_\mu}^2 \lrr{\eta^{(L+1)}}^2\Ee{\lrr{1-\frac{1}{n_{L}}}\lrr{\partial_\mu z_{1}^{(L)}z_{2}^{(L)})}^2}\\
    &=\frac{\lrr{\eta^{(L+1)}}^2}{2}\left[\Y{L}{(\partial_\mu z)^2,z^2}+\frac{1}{n_L}\Y{L}{(z\partial_\mu z)^2}\right].
\end{align*}
Finally, the fourth sum is
\begin{align*}
    \sum_{\mu,\nu\leq L} \Ee{\lrr{\widehat{\eta}_\mu\widehat{\eta}_\nu}^2\lrr{\partial_{\mu\nu} z_{1}^{(L+1)}z_{1}^{(L+1)}}^2}&= \sum_{\mu,\nu\leq L} \Ee{\lrr{\widehat{\eta}_\mu\widehat{\eta}_\nu}^2\lrr{\sum_{j_1,j_2=1}^{N_L}W_{1j_1}^{(L+1)}W_{1j_2}^{(L+1)}\sigma(z_{j_1}^{(L)})\partial_{\mu\nu} \sigma(z_{j_2}^{(L)})}^2}\\
    &=\sum_{\mu,\nu\leq L} \Ee{\lrr{\widehat{\eta}_\mu\widehat{\eta}_\nu}^2\frac{4}{n_L^2}\sum_{j_1,j_2=1}^{N_L}\lrr{\sigma(z_{j_1}^{(L)})\partial_{\mu\nu} \sigma(z_{j_2}^{(L)})}^2}\\
    &+2\sum_{\mu,\nu\leq L} \Ee{\lrr{\widehat{\eta}_\mu\widehat{\eta}_\nu}^2\frac{4}{n_L^2}\sum_{j_1,j_2=1}^{N_L}\sigma(z_{j_1}^{(L)})\partial_{\mu\nu} \sigma(z_{j_1}^{(L)})\sigma(z_{j_2}^{(L)})\partial_{\mu\nu} \sigma(z_{j_2}^{(L)})}.
\end{align*}
To proceed we evaluate the first term as follows:
\begin{align*}
    \sum_{\mu,\nu\leq L} \Ee{\lrr{\widehat{\eta}_\mu\widehat{\eta}_\nu}^2\frac{4}{n_L^2}\sum_{j_1,j_2=1}^{N_L}\lrr{\sigma(z_{j_1}^{(L)})\partial_{\mu\nu} \sigma(z_{j_2}^{(L)})}^2}&=\sum_{\mu,\nu\leq L} \Ee{\lrr{\widehat{\eta}_\mu\widehat{\eta}_\nu}^2\frac{4}{n_L}\lrr{\sigma(z_{1}^{(L)})\partial_{\mu\nu} \sigma(z_{1}^{(L)})}^2}\\
    &+\sum_{\mu,\nu\leq L} \Ee{\lrr{\widehat{\eta}_\mu\widehat{\eta}_\nu}^24\lrr{1-\frac{1}{n_L}}\lrr{\sigma(z_{1}^{(L)})\partial_{\mu\nu} \sigma(z_{2}^{(L)})}^2}\\
    &=\Y{L}{(\partial_{\mu\nu}z)^2,z^2}+\frac{1}{n_L} \Y{L}{(z\partial_{\mu\nu}z)^2}.
\end{align*}
Further, the second term is
\begin{align*}
    &\sum_{\mu,\nu\leq L} \Ee{\lrr{\widehat{\eta}_\mu\widehat{\eta}_\nu}^2\frac{4}{n_L^2}\sum_{j_1,j_2=1}^{N_L}\sigma(z_{j_1}^{(L)})\partial_{\mu\nu} \sigma(z_{j_1}^{(L)})\sigma(z_{j_2}^{(L)})\partial_{\mu\nu} \sigma(z_{j_2}^{(L)})}\\
    &\qquad = \sum_{\mu,\nu\leq L} \Ee{\lrr{\widehat{\eta}_\mu\widehat{\eta}_\nu}^2\frac{4}{n_L}\lrr{\sigma(z_{1}^{(L)})\partial_{\mu\nu} \sigma(z_{1}^{(L)})}^2}\\
    &\qquad + \sum_{\mu,\nu\leq L} \Ee{\lrr{\widehat{\eta}_\mu\widehat{\eta}_\nu}^24\lrr{1-\frac{1}{n_L}}\sigma(z_{1}^{(L)})\partial_{\mu\nu} \sigma(z_{1}^{(L)})\sigma(z_{2}^{(L)})\partial_{\mu\nu} \sigma(z_{2}^{(L)})}\\
    &\qquad =\Y{L}{z\partial_{\mu\nu }z,z\partial_{\mu\nu }z}+\frac{1}{n_L}\Y{L}{(z\partial_{\mu\nu}z)^2}
\end{align*}
Putting all this together yields
\begin{align*}
   \sum_{\mu,\nu\leq L+1} \Ee{\lrr{\widehat{\eta}_\mu\widehat{\eta}_\nu}^2\lrr{\partial_{\mu\nu} z_{1}^{(L+1)}z_{1}^{(L+1)}}^2}&= \lrr{\eta^{(L+1)}}^2\left[\Y{L}{(\partial_\mu z)^2,z^2}+\frac{1}{n_L}\Y{L}{(z\partial_\mu z)^2}\right]\\
   &+\Y{L}{(\partial_{\mu\nu}z)^2,z^2}+\frac{1}{n_L} \Y{L}{(z\partial_{\mu\nu}z)^2}\\
   &+2\Y{L}{z\partial_{\mu\nu }z,z\partial_{\mu\nu }z}+\frac{2}{n_L}\Y{L}{(z\partial_{\mu\nu}z)^2},
\end{align*}
which is precisely the statement of \eqref{eq:rec1}. Next, we establish \eqref{eq:rec2} in a similar manner. We have
\begin{align*}
    \sum_{\mu,\nu\leq L+1}  \Ee{\lrr{\widehat{\eta}_\mu\widehat{\eta}_\nu}^2\lrr{\partial_{\mu\nu} z_{1}^{(L+1)}}^2}&=\sum_{\mu,\nu\in L+1}  \Ee{\lrr{\widehat{\eta}_\mu\widehat{\eta}_\nu}^2\lrr{\partial_{\mu\nu}z_{1}^{(L+1)}}^2}\\
    &+\sum_{\mu\leq L,\nu\in L+1}  \Ee{\lrr{\widehat{\eta}_\mu\widehat{\eta}_\nu}^2\lrr{\partial_{\mu\nu} z_{1}^{(L+1)}}^2}\\
    &+\sum_{\mu\in L+1,\nu\leq L}  \Ee{\lrr{\widehat{\eta}_\mu\widehat{\eta}_\nu}^2\lrr{\partial_{\mu\nu} z_{1}^{(L+1)}}^2}\\
    &+\sum_{\mu,\nu\leq L}  \Ee{\lrr{\widehat{\eta}_\mu\widehat{\eta}_\nu}^2\lrr{\partial_{\mu\nu} z_{1}^{(L+1)}}^2}.
\end{align*}
Again the first sum vanishes since $\partial_{\mu\nu}z_{1}^{(L+1)}=0$ if $\mu,\nu$ are weights in the final layer. Next, the second and third terms are equal and can be written as follows:
\begin{align*}
    \sum_{\mu\leq L,\nu\in L+1}  \Ee{\lrr{\widehat{\eta}_\mu\widehat{\eta}_\nu}^2\lrr{\partial_{\mu\nu} z_{1}^{(L+1)}}^2}&=\lrr{\eta^{(L+1)}}^2\sum_{\mu\leq L} \frac{1}{n_L}\sum_{j=1}^{n_L} \Ee{\lrr{\widehat{\eta}_\mu}^2\lrr{\partial_{\mu} \sigma(z_j^{(L)})}^2}\\
    &=\frac{1}{2}\lrr{\eta^{(L+1)}}^2\Y{L}{(\partial_\mu z)^2}.
\end{align*}
Finally, the fourth term in the sum can be rewritten in the following manner
\begin{align*}
    \sum_{\mu,\nu\leq L}  \Ee{\lrr{\widehat{\eta}_\mu\widehat{\eta}_\nu}^2\lrr{\partial_{\mu\nu} z_{1}^{(L+1)}}^2}&=\sum_{\mu,\nu\leq L}  \Ee{\lrr{\widehat{\eta}_\mu\widehat{\eta}_\nu}^2\lrr{\sum_{j=1}^{n_L}W_{1j}^{(L+1)}\partial_{\mu\nu} \sigma(z_j^{(L)})}^2}\\
    &=\sum_{\mu,\nu\leq L}  \Ee{\lrr{\widehat{\eta}_\mu\widehat{\eta}_\nu}^2\frac{2}{n_L}\sum_{j=1}^{n_L}\lrr{\partial_{\mu\nu} \sigma(z_j^{(L)})}^2}\\
    &=\Y{L}{(\partial_{\mu\nu}z)^2}.
\end{align*}
Hence, altogether, we find
\begin{align*}
    \sum_{\mu,\nu\leq L+1}  \Ee{\lrr{\widehat{\eta}_\mu\widehat{\eta}_\nu}^2\lrr{\partial_{\mu\nu} z_{1}^{(L+1)}}^2}&=\lrr{\eta^{(L+1)}}^2\Y{L}{(\partial_\mu z)^2} + \Y{L}{(\partial_{\mu\nu}z)^2},
\end{align*}
which is the statement of \eqref{eq:rec2}. Next, we establish \eqref{eq:rec3}. We have
\begin{align*}
    &\sum_{\mu,\nu\leq L+1} \Ee{\lrr{\widehat{\eta}_\mu\widehat{\eta}_\nu}^2\partial_{\mu\nu} z_{1}^{(L+1)}\partial_\mu z_{1}^{(L+1)}\partial_\nu z_{1}^{(L+1)}z_{1}^{(L+1)}}\\
    &\qquad =\sum_{\mu,\nu\in L+1} \Ee{\lrr{\widehat{\eta}_\mu\widehat{\eta}_\nu}^2\partial_{\mu\nu} z_{1}^{(L+1)}\partial_\mu z_{1}^{(L+1)}\partial_\nu z_{1}^{(L+1)}z_{1}^{(L+1)}}\\
    &\qquad +\sum_{\mu\leq L,\nu\in L+1} \Ee{\lrr{\widehat{\eta}_\mu\widehat{\eta}_\nu}^2\partial_{\mu\nu} z_{1}^{(L+1)}\partial_\mu z_{1}^{(L+1)}\partial_\nu z_{1}^{(L+1)}z_{1}^{(L+1)}}\\
    &\qquad +\sum_{\mu\in L+1,\nu\leq L} \Ee{\lrr{\widehat{\eta}_\mu\widehat{\eta}_\nu}^2\partial_{\mu\nu} z_{1}^{(L+1)}\partial_\mu z_{1}^{(L+1)}\partial_\nu z_{1}^{(L+1)}z_{1}^{(L+1)}}\\
    &\qquad +\sum_{\mu,\nu\leq L} \Ee{\lrr{\widehat{\eta}_\mu\widehat{\eta}_\nu}^2\partial_{\mu\nu} z_{1}^{(L+1)}\partial_\mu z_{1}^{(L+1)}\partial_\nu z_{1}^{(L+1)}z_{1}^{(L+1)}}.
\end{align*}
The first sum vanishes since $\partial_{\mu\nu} z_{1}^{(L+1)}=0$ when $\mu,\nu$ are weights in the final layer. The second and third terms are again the same and equal
\begin{align*}
   &\sum_{\mu\leq L,\nu\in L+1} \Ee{\lrr{\widehat{\eta}_\mu\widehat{\eta}_\nu}^2\partial_{\mu\nu} z_{1}^{(L+1)}\partial_\mu z_{1}^{(L+1)}\partial_\nu z_{1}^{(L+1)}z_{1}^{(L+1)}}\\
   &=\lrr{\eta^{(L+1)}}^2\sum_{\mu\leq L} \Ee{\lrr{\widehat{\eta}_\mu}^2\frac{1}{n_L}\sum_{j_1=1}^{n_L}\partial_{\mu} \sigma(z_{j_1}^{(L)})\partial_\mu z_{1}^{(L+1)}\sigma(z_{j_1}^{(L)})z_{1}^{(L+1)}}\\
   &=\lrr{\eta^{(L+1)}}^2\sum_{\mu\leq L} \Ee{\lrr{\widehat{\eta}_\mu}^2\frac{1}{n_L}\sum_{j_1=1}^{n_L}\partial_{\mu} \sigma(z_{j_1}^{(L)})\sigma(z_{j_1}^{(L)})\sum_{j_2=1}^{n_L}W_{1j_2}^{(L+1)}W_{1j_2}^{(L+1)}\partial_\mu \sigma(z_{j_2}^{(L)})\sigma(z_{j_2}^{(L)})}\\
   &=\lrr{\eta^{(L+1)}}^2\sum_{\mu\leq L} \Ee{\lrr{\widehat{\eta}_\mu}^2\frac{2}{n_L^2}\sum_{j_1,j_2=1}^{n_L}\partial_{\mu} \sigma(z_{j_1}^{(L)})\sigma(z_{j_1}^{(L)})\partial_\mu \sigma(z_{j_2}^{(L)})\sigma(z_{j_2}^{(L)})}\\
   &=\lrr{\eta^{(L+1)}}^2\sum_{\mu\leq L} \Ee{\lrr{\widehat{\eta}_\mu}^2\frac{2}{n_L}\lrr{\partial_{\mu} \sigma(z_{1}^{(L)})\sigma(z_{1}^{(L)})}^2}\\
   &+\lrr{\eta^{(L+1)}}^2\sum_{\mu\leq L} \Ee{\lrr{\widehat{\eta}_\mu}^22\lrr{1-\frac{1}{n_L}}\partial_{\mu} \sigma(z_{1}^{(L)})\sigma(z_{1}^{(L)})\partial_\mu \sigma(z_{2}^{(L)})\sigma(z_{2}^{(L)})}\\
   &=\frac{1}{2}\lrr{\eta^{(L+1)}}^2\lrr{\Y{L}{z\partial_\mu z,z\partial_\mu z}+ \frac{1}{n_L}\Y{L}{(z\partial_\mu z)^2}}.
\end{align*}
Finally, the fourth term is
\begin{align*}
    &\sum_{\mu,\nu\leq L} \Ee{\lrr{\widehat{\eta}_\mu\widehat{\eta}_\nu}^2\partial_{\mu\nu} z_{1}^{(L+1)}\partial_\mu z_{1}^{(L+1)}\partial_\nu z_{1}^{(L+1)}z_{1}^{(L+1)}}\\
    &\qquad =\sum_{\mu,\nu\leq L} \Ee{\lrr{\widehat{\eta}_\mu\widehat{\eta}_\nu}^2\sum_{j_1,j_2,j_3,j_4=1}^{n_L} W_{1j_1}^{(L+1)}W_{1j_2}^{(L+1)}W_{1j_3}^{(L+1)}W_{1j_4}^{(L+1)}\partial_{\mu\nu} \sigma(z_{j_1}^{(L)})\partial_\mu \sigma(z_{j_2}^{(L)})\partial_\nu \sigma(z_{j_3}^{(L)})\sigma(z_{j_4}^{(L)})}\\
    &\qquad =2\sum_{\mu,\nu\leq L} \Ee{\lrr{\widehat{\eta}_\mu\widehat{\eta}_\nu}^2\frac{4}{n_L^2}\sum_{j_1,j_2}^{n_L} \partial_{\mu\nu} \sigma(z_{j_1}^{(L)})\partial_\mu \sigma(z_{j_1}^{(L)})\partial_\nu \sigma(z_{j_2}^{(L)})\sigma(z_{j_2}^{(L)})}\\
    &\qquad +\sum_{\mu,\nu\leq L} \Ee{\lrr{\widehat{\eta}_\mu\widehat{\eta}_\nu}^2\frac{4}{n_L^2}\sum_{j_1,j_2}^{n_L} \partial_{\mu\nu} \sigma(z_{j_1}^{(L)})\sigma(z_{j_1}^{(L)})\partial_\mu \sigma(z_{j_2}^{(L)})\partial_\nu \sigma(z_{j_2}^{(L)})}\\
    &\qquad =2\Y{L}{\partial_{\mu \nu}z \partial_\mu z, z\partial_\nu z }+\Y{L}{z\partial_{\mu\nu}z, \partial_\mu z\partial_\nu z}+\frac{2}{n_L}\Y{L}{z\partial_{\mu \nu}z \partial_\mu z \partial_\nu z}.
\end{align*}
Putting this all together yields
\begin{align*}
     &\sum_{\mu,\nu\leq L+1} \Ee{\lrr{\widehat{\eta}_\mu\widehat{\eta}_\nu}^2\lrr{\partial_{\mu\nu} z_{1}^{(L+1)}z_{1}^{(L+1)}}^2}\\
     &\qquad = \lrr{\eta^{(L+1)}}^2\lrr{\Y{L}{z\partial_\mu z,z\partial_\mu z}+ \frac{1}{n_L}\Y{L}{(z\partial_\mu z)^2}}\\
     &\qquad +2\Y{L}{\partial_{\mu \nu}z \partial_\mu z, z\partial_\nu z }+\Y{L}{z\partial_{\mu\nu}z, \partial_\mu z\partial_\nu z}+\frac{2}{n_L}\Y{L}{z\partial_{\mu \nu}z \partial_\mu z \partial_\nu z},
\end{align*}
which precisely the statement of \eqref{eq:rec3}. Finally, it remains to check \eqref{eq:rec4}. We have
\begin{align*}
    \sum_{\mu,\nu\leq L+1}\Ee{\lrr{\widehat{\eta}_\mu\widehat{\eta}_\nu}^2\lrr{\partial_{\mu} z_{1}^{(L+1)}\partial_\nu z_{1}^{(L+1)}}^2}&=\sum_{\mu,\nu\in L+1}\Ee{\lrr{\widehat{\eta}_\mu\widehat{\eta}_\nu}^2\lrr{\partial_{\mu} z_{1}^{(L+1)}\partial_\nu z_{1}^{(L+1)}}^2}\\
    &+\sum_{\mu\leq L,\nu\in L+1}\Ee{\lrr{\widehat{\eta}_\mu\widehat{\eta}_\nu}^2\lrr{\partial_{\mu} z_{1}^{(L+1)}\partial_\nu z_{1}^{(L+1)}}^2}\\
    &+\sum_{\mu\in L+1,\nu\leq  L}\Ee{\lrr{\widehat{\eta}_\mu\widehat{\eta}_\nu}^2\lrr{\partial_{\mu} z_{1}^{(L+1)}\partial_\nu z_{1}^{(L+1)}}^2}\\
    &+\sum_{\mu,\nu\leq  L}\Ee{\lrr{\widehat{\eta}_\mu\widehat{\eta}_\nu}^2\lrr{\partial_{\mu} z_{1}^{(L+1)}\partial_\nu z_{1}^{(L+1)}}^2}.
\end{align*}
The first term equals
\begin{align*}
   \sum_{\mu,\nu\in L+1}\Ee{\lrr{\widehat{\eta}_\mu\widehat{\eta}_\nu}^2\lrr{\partial_{\mu} z_{1}^{(L+1)}\partial_\nu z_{1}^{(L+1)}}^2}& = \lrr{\eta^{(L+1)}}^4 \Ee{\frac{1}{n_L^2}\sum_{j_1,j_2=1}^{n_L} \lrr{\sigma(z_{j_1}^{(L)})\sigma(z_{j_2}^{(L)})}^2}\\
   &=\frac{1}{2}\lrr{\eta^{(L+1)}}^4 \lrr{\Ee{\lrr{\frac{1}{n_L}\norm{z^{(L)}}_2^2}^2}+\frac{1}{n_L}\Ee{\frac{1}{n_L}\norm{z^{(L)}}_4^4}}.
\end{align*}
The second and third terms are the same and equal
\begin{align*}
    &\sum_{\mu\leq L,\nu\in L+1}\Ee{\lrr{\widehat{\eta}_\mu\widehat{\eta}_\nu}^2\lrr{\partial_{\mu} z_{1}^{(L+1)}\partial_\nu z_{1}^{(L+1)}}^2}\\
    &\qquad =\lrr{\eta^{(L+1)}}^2\sum_{\mu\leq L}\Ee{\lrr{\widehat{\eta}_\mu}^2\frac{1}{n_L}\sum_{j_1,j_2,j_3=1}^{n_L}\lrr{\sigma(z_{j_1}^{(\ell)})}^2 W_{1,j_2}^{(L+1)}W_{1,j_3}^{(L+1)}\partial_\mu \sigma(z_{j_2}^{(L)})\partial_\mu \sigma(z_{j_3}^{(L)})}\\
    &\qquad =\lrr{\eta^{(L+1)}}^2\sum_{\mu\leq L}\Ee{\lrr{\widehat{\eta}_\mu}^2\frac{2}{n_L^2}\sum_{j_1,j_2=1}^{n_L}\lrr{\sigma(z_{j_1}^{(\ell)}) \partial_\mu \sigma(z_{j_2}^{(L)})}^2}\\
    &=\lrr{\eta^{(L+1)}}^2\left[\Y{L}{z^2,\lrr{ \partial_\mu z}^2}+\frac{1}{n_L}\Y{L}{\lrr{z\partial_\mu z}^2}\right]
\end{align*}
The fourth term equals
\begin{align*}
    &\sum_{\mu,\nu\leq L}\Ee{\lrr{\widehat{\eta}_\mu\widehat{\eta}_\nu}^2\lrr{\partial_{\mu} z_{1}^{(L+1)}\partial_\nu z_{1}^{(L+1)}}^2}\\
    &\qquad =\sum_{\mu,\nu\leq L}\Ee{\lrr{\widehat{\eta}_\mu\widehat{\eta}_\nu}^2\sum_{j_1,j_2,j_3,j_4=1}^{n_L}W_{1j_1}^{(L+1)}W_{1j_2}^{(L+1)}W_{1j_3}^{(L+1)}W_{1j_4}^{(L+1)}\partial_{\mu} \sigma(z_{j_1}^{(L)})\partial_\nu \sigma(z_{j_2}^{(L+1)})\partial_{\mu} \sigma(z_{j_3}^{(L)})\partial_\nu \sigma(z_{j_4}^{(L+1)})}\\
    &\qquad =2\sum_{\mu,\nu\leq L}\Ee{\lrr{\widehat{\eta}_\mu\widehat{\eta}_\nu}^2\frac{4}{n_L^2}\sum_{j_1,j_2=1}^{n_L}\partial_{\mu} \sigma(z_{j_1}^{(L)})\partial_\nu \sigma(z_{j_1}^{(L+1)})\partial_{\mu} \sigma(z_{j_2}^{(L)})\partial_\nu \sigma(z_{j_2}^{(L+1)})}\\
    &\qquad + \sum_{\mu,\nu\leq L}\Ee{\lrr{\widehat{\eta}_\mu\widehat{\eta}_\nu}^2\frac{4}{n_L^2}\sum_{j_1,j_2=1}^{n_L}\lrr{\partial_{\mu} \sigma(z_{j_1}^{(L)})\partial_\nu  \sigma(z_{j_2}^{(L+1)})}^2}\\
    &\qquad =2\Y{L}{\partial_\mu z \partial_\nu z,\partial_\mu z \partial_\nu z}+\Y{L}{\lrr{\partial_\mu z}^2,\lrr{\partial_\nu z}^2}+\frac{3}{n_L}\Y{L}{\lrr{\partial_\mu z \partial_\nu z}^2}.
\end{align*}
Putting this all together yields
\begin{align*}
    & \sum_{\mu,\nu\leq L+1}\Ee{\lrr{\widehat{\eta}_\mu\widehat{\eta}_\nu}^2\lrr{\partial_{\mu} z_{1}^{(L+1)}\partial_\nu z_{1}^{(L+1)}}^2}\\
    &\qquad = \frac{1}{2}\lrr{\eta^{(L+1)}}^4 \lrr{\Ee{\lrr{\frac{1}{n_L}\norm{z^{(L)}}_2^2}^2}+\frac{1}{n_L}\Ee{\frac{1}{n_L}\norm{z^{(L)}}_4^4}}\\
    &\qquad + 2\lrr{\eta^{(L+1)}}^2\left[\Y{L}{z^2,\lrr{ \partial_\mu z}^2}+\frac{1}{n_L}\Y{L}{\lrr{z\partial_\mu z}^2}\right]\\
    &\qquad +2\Y{L}{\partial_\mu z \partial_\nu z,\partial_\mu z \partial_\nu z}+\Y{L}{\lrr{\partial_\mu z}^2,\lrr{\partial_\nu z}^2}+\frac{3}{n_L}\Y{L}{\lrr{\partial_\mu z \partial_\nu z}^2},
\end{align*}
which is precisely the statement of \eqref{eq:rec4}.
\end{proof}

In particular, combining \eqref{eq:HS-def} and \eqref{eq:d2L} with the result of the preceding Lemma yields the following result.
\begin{corollary}\label{cor:H-Y}
We have,
\begin{align*}
    &\Ee{\norm{H_{\text{eff}}}_{HS}^2}\\
    &\qquad =\lrr{\eta^{(L+1)}}^2\left[\Y{L}{(\partial_\mu z)^2,z^2}+\frac{1}{n_L}\Y{L}{(z\partial_\mu z)^2}\right]+\Y{L}{(\partial_{\mu\nu}z)^2,z^2}+\frac{1}{n_L} \Y{L}{(z\partial_{\mu\nu}z)^2}\\
  &\qquad +2\Y{L}{z\partial_{\mu\nu }z,z\partial_{\mu\nu }z}+\frac{2}{n_L}\Y{L}{(z\partial_{\mu\nu}z)^2}\\
&\qquad +\lrr{\eta^{(L+1)}}^2 \Y{L}{\lrr{\partial_\mu z}^2}
      +\Y{L}{\lrr{\partial_{\mu\nu} z}^2}\\
 &\qquad +2\lrr{\eta^{(L+1)}}^2\lrr{\Y{L}{z\partial_\mu z,z\partial_\mu z}+ \frac{1}{n_L}\Y{L}{(z\partial_\mu z)^2}}\\
       &\qquad +4\Y{L}{\partial_{\mu \nu}z \partial_\mu z, z\partial_\nu z }+2\Y{L}{z\partial_{\mu\nu}z, \partial_\mu z\partial_\nu z}+\frac{6}{n_L}\Y{L}{z\partial_{\mu \nu}z \partial_\mu z \partial_\nu z}\\
    &\qquad +\frac{1}{2}\lrr{\eta^{(L+1)}}^4 \lrr{\Ee{\lrr{\frac{1}{n_L}\norm{z^{(L)}}_2^2}^2}+\frac{1}{n_L}\Ee{\frac{1}{n_L}\norm{z^{(L)}}_4^4}}\\
    &\notag \qquad + \lrr{\eta^{(L+1)}}^2\left[\Y{L}{z^2,\lrr{ \partial_\mu z}^2}+\frac{1}{n_L}\Y{L}{\lrr{z\partial_\mu z}^2}\right]\\
    &\qquad +2\Y{L}{\partial_\mu z \partial_\nu z,\partial_\mu z \partial_\nu z}+\Y{L}{\lrr{\partial_\mu z}^2,\lrr{\partial_\nu z}^2}+\frac{3}{n_L}\Y{L}{\lrr{\partial_\mu z \partial_\nu z}^2}.
\end{align*}
\end{corollary}
\subsection{Self-Consistent Recursions for $Y^{(\ell)}$'s}

Our next task is to develop and then solve self-consistent recursions for those $Y$'s that contain only one $\mu$. 
\begin{lemma}\label{lem:Ymu-recs}
We have
\begin{align}
 \label{eq:Ymu-rec1}   \Y{\ell+1}{(\partial_\mu z)^2} &= \frac{1}{2}\lrr{\eta^{(\ell+1)}}^2 \frac{1}{n_0}\norm{x}_2^2 + \Y{\ell}{(\partial_\mu z)^2}\\
\label{eq:Ymu-rec2}    \Y{\ell+1}{(\partial_\mu z)^2,z^2} &= \lrr{{\eta}^{(\ell+1)}}^2\Ee{\lrr{\frac{1}{n_\ell}||\sigma^{(\ell)}||^2}^2}\\
\notag &+\Y{\ell}{z^2, (\partial_\mu z)^2}  + \frac{2}{n_{\ell+1}}\Y{\ell}{z\partial_\mu, z\partial_\mu}+ \frac{1}{n_\ell}\lrr{1+\frac{2}{n_{\ell+1}}}\Y{\ell}{(z\partial_\mu z)^2}\\
\label{eq:Ymu-rec3}    \Y{\ell+1}{z\partial_\mu z,z\partial_\mu z} &= 2\frac{\lrr{\eta^{(\ell+1)}}^2}{n_{\ell+1}}\Ee{\lrr{\frac{1}{n_\ell}||\sigma^{(\ell)}||^2}^2}\\
\notag    &+\lrr{1+\frac{1}{n_{\ell+1}}}\Y{\ell}{z\partial_\mu z, z\partial_\mu z} + \frac{1}{n_\ell}\lrr{1+\frac{2}{n_{\ell+1}}} \Y{\ell}{(z\partial_\mu z)^2} + \frac{1}{n_{\ell+1}} \Y{\ell}{z^2, \partial_\mu z^2}\\
\label{eq:Ymu-rec4}    \Y{\ell+1}{(z\partial_\mu z)^2}&= 2\lrr{\eta^{(\ell+1)}}^2\Ee{\lrr{\frac{1}{n_\ell}||\sigma^{(\ell)}||^2}^2}\\
\notag &+2\Y{\ell}{z\partial_\mu z, z\partial_\mu z} + \Y{\ell}{z^2,\partial_\mu z^2} + \frac{3}{n_\ell} \Y{\ell}{(z\partial_\mu z)^2}.
\end{align}
\end{lemma}
\begin{proof}
We start with \eqref{eq:Ymu-rec1}. We have
\begin{align*}
    \Y{\ell+1}{(\partial_\mu z)^2} &= \Ee{\sum_{\mu\in  \ell+1}\lrr{\widehat{\eta}_\mu}^2\frac{1}{n_{\ell+1}} \sum_{j=1}^{n_{\ell+1}} \lrr{\partial_\mu z_{j}^{(\ell+1)}}^2}\\
    &+\Ee{\sum_{\mu\leq  \ell}\lrr{\widehat{\eta}_\mu}^2\frac{1}{n_{\ell+1}} \sum_{j=1}^{n_{\ell+1}} \lrr{\partial_\mu z_{j}^{(\ell+1)}}^2}.
\end{align*}
The first term is 
\begin{align*}
    \Ee{\sum_{\mu\in  \ell+1}\lrr{\widehat{\eta}_\mu}^2\frac{1}{n_{\ell+1}} \sum_{j=1}^{n_{\ell+1}} \lrr{\partial_\mu z_{j}^{(\ell+1)}}^2}&= \lrr{\eta^{(\ell+1)}}^2\Ee{\frac{1}{n_\ell}\norm{\sigma^{(\ell)}}^2}=\frac{\norm{x}^2}{2n_0}\lrr{\eta^{(\ell+1)}}^2.
\end{align*}
Next, for the second term, we have
\begin{align*}
    \Ee{\sum_{\mu\leq  \ell}\lrr{\widehat{\eta}_\mu}^2\frac{1}{n_{\ell+1}} \sum_{j=1}^{n_{\ell+1}} \lrr{\partial_\mu z_{j}^{(\ell+1)}}^2}&=\Ee{\sum_{\mu\leq  \ell}\lrr{\widehat{\eta}_\mu}^2\frac{1}{n_{\ell+1}} \sum_{j=1}^{n_{\ell+1}} \sum_{k_1,k_2=1}^{n_L} W_{jk_1}^{(\ell+1)} W_{jk_2}^{(\ell+1)} \partial_\mu \sigma(z_{k_1}^{(\ell)})\sigma(z_{k_2}^{(\ell)})}\\
    &=\Ee{\sum_{\mu\leq  \ell}\lrr{\widehat{\eta}_\mu}^2\frac{2}{n_\ell} \sum_{k=1}^{n_L}  \lrr{\partial_\mu \sigma(z_{k}^{(\ell)})}^2}\\
    &=\Y{\ell}{(\partial_\mu z)^2}.
\end{align*}
Combining the preceding expressions yields
\begin{align*}
\Y{\ell+1}{(\partial_\mu z)^2} &=\frac{\norm{x}^2}{2n_0}\lrr{\eta^{(\ell+1)}}^2+\Y{\ell}{(\partial_\mu z)^2},
\end{align*}    
which is precisely \eqref{eq:Ymu-rec1}. Next, let us derive \eqref{eq:Ymu-rec2}. We have
\begin{align*}
     \Y{\ell+1}{(\partial_\mu z)^2,z^2}&= \Ee{\sum_{\mu\in \ell+1} \lrr{\widehat{\eta}_{\mu}}^2 \frac{1}{n_{\ell+1}}\sum_{j=1}^{n_{\ell+1}} \lrr{z_j^{(\ell+1)}\partial_\mu z_j^{(\ell+1)}}^2}\\
     &+\Ee{\sum_{\mu\leq \ell} \lrr{\widehat{\eta}_{\mu}}^2 \frac{1}{n_{\ell+1}}\sum_{j=1}^{n_{\ell+1}} \lrr{z_j^{(\ell+1)}\partial_\mu z_j^{(\ell+1)}}^2}.
\end{align*}
The first term is 
\begin{align*}
     \Ee{\sum_{\mu\in \ell+1} \lrr{\widehat{\eta}_{\mu}}^2 \frac{1}{n_{\ell+1}}\sum_{j=1}^{n_{\ell+1}} \lrr{z_j^{(\ell+1)}\partial_\mu z_j^{(\ell+1)}}^2}&= \lrr{{\eta}^{(\ell+1)}}^2\Ee{\frac{1}{n_\ell}\sum_{j_1=1}^{n_\ell}\sigma(z_{j_1}^{(\ell)})^2 \frac{1}{n_{\ell+1}}\sum_{j=1}^{n_{\ell+1}} \lrr{z_j^{(\ell+1)}}^2}\\
     &=\lrr{{\eta}^{(\ell+1)}}^2\Ee{\frac{1}{n_\ell^2}\sum_{j_2=1}^{n_\ell}\sigma(z_{j_1}^{(\ell)})^2\sigma(z_{j_2}^{(\ell)})^2 }\\
     &=\lrr{{\eta}^{(\ell+1)}}^2\Ee{\lrr{\frac{1}{n_\ell}||\sigma^{(\ell)}||^2}^2}.
\end{align*}
In contrast, using Wick's theorem, the second term is
\begin{align*}
    &\Ee{\sum_{\mu\leq \ell} \lrr{\widehat{\eta}_{\mu}}^2 \frac{1}{n_{\ell+1}}\sum_{j=1}^{n_{\ell+1}} \lrr{z_j^{(\ell+1)}\partial_\mu z_j^{(\ell+1)}}^2}\\
    &\qquad = \Ee{\sum_{\mu\leq \ell} \lrr{\widehat{\eta}_{\mu}}^2 \frac{1}{n_{\ell+1}}\sum_{j=1}^{n_{\ell+1}} \sum_{j_1,j_2,j_3,j_4=1}^{n_\ell} W_{jj_1}^{(\ell+1)}W_{jj_2}^{(\ell+1)}W_{jj_3}^{(\ell+1)}W_{jj_4}^{(\ell+1)}\sigma(z_{j_1}^{(\ell)})\sigma(z_{j_2}^{(\ell)})\partial_\mu \sigma(z_{j_3}^{(\ell)})\partial_\mu \sigma(z_{j_4}^{(\ell)})}\\
    &\qquad = \Ee{\sum_{\mu\leq \ell} \lrr{\widehat{\eta}_{\mu}}^2 \frac{4}{n_\ell^2}\sum_{j_1,j_2}^{n_\ell} \lrr{\sigma(z_{j_1}^{(\ell)})\partial_\mu \sigma(z_{j_2}^{(\ell)})}^2 + 2\sigma(z_{j_1}^{(\ell)})\partial_\mu \sigma(z_{j_1}^{(\ell)})\sigma(z_{j_2}^{(\ell)})\partial_\mu \sigma(z_{j_2}^{(\ell)})}\\
    &\qquad = \Y{\ell}{z^2, (\partial_\mu z)^2}  + 2\Y{\ell}{z\partial_\mu, z\partial_\mu}+ \frac{3}{n_\ell} \Y{\ell}{(z\partial_\mu z)^2}.
\end{align*}
Hence, altogether, we find
\begin{align*}
     \Y{\ell+1}{(\partial_\mu z)^2,z^2}&=\lrr{{\eta}^{(\ell+1)}}^2\Ee{\lrr{\frac{1}{n_\ell}||\sigma^{(\ell)}||^2}^2}+\Y{\ell}{z^2, (\partial_\mu z)^2}  + 2\Y{\ell}{z\partial_\mu, z\partial_\mu}+ \frac{3}{n_\ell} \Y{\ell}{(z\partial_\mu z)^2},
\end{align*}
which is precisely \eqref{eq:Ymu-rec2}. Next, we derive \eqref{eq:Ymu-rec3}. We have
\begin{align*}
    \Y{\ell+1}{z\partial_\mu z,z\partial_\mu z}&= \Ee{\sum_{\mu\in \ell+1}\lrr{\widehat{\eta}_\mu}^2 \frac{1}{n_{\ell+1}^2}\sum_{j_1,j_2=1}^{n_{\ell+1}} z_{j_1}^{(\ell+1)}z_{j_2}^{(\ell+1)}\partial_\mu z_{j_1}^{(\ell+1)}\partial_\mu z_{j_2}^{(\ell+1)}}\\
    &= \Ee{\sum_{\mu\leq \ell}\lrr{\widehat{\eta}_\mu}^2 \frac{1}{n_{\ell+1}^2}\sum_{j_1,j_2=1}^{n_{\ell+1}} z_{j_1}^{(\ell+1)}z_{j_2}^{(\ell+1)}\partial_\mu z_{j_1}^{(\ell+1)}\partial_\mu z_{j_2}^{(\ell+1)}}.
\end{align*}
The first term equals
\begin{align*}
    \Ee{\sum_{\mu\in \ell+1}\lrr{\widehat{\eta}_\mu}^2 \frac{1}{n_{\ell+1}^2}\sum_{j_1,j_2=1}^{n_{\ell+1}} z_{j_1}^{(\ell+1)}z_{j_2}^{(\ell+1)}\partial_\mu z_{j_1}^{(\ell+1)}\partial_\mu z_{j_2}^{(\ell+1)}}&= \frac{\lrr{\eta^{(\ell+1)}}^2}{n_{\ell+1}} \Ee{\frac{1}{n_{\ell}}\sum_{j_1=1}^{n_{\ell}}\sigma(z_{j_1}^{(\ell)^2})\frac{1}{n_{\ell+1}}\sum_{j_2=1}^{n_{\ell+1}}\lrr{z_{j_2}^{(\ell+1)}}^2}\\
    &= 2\frac{\lrr{\eta^{(\ell+1)}}^2}{n_{\ell+1}}\Ee{\lrr{\frac{1}{n_\ell}||\sigma^{(\ell)}||^2}^2}.
\end{align*}
The second term is
\begin{align*}
    &\Ee{\sum_{\mu\leq \ell}\lrr{\widehat{\eta}_\mu}^2 \frac{1}{n_{\ell+1}^2}\sum_{j_1,j_2=1}^{n_{\ell+1}} z_{j_1}^{(\ell+1)}z_{j_2}^{(\ell+1)}\partial_\mu z_{j_1}^{(\ell+1)}\partial_\mu z_{j_2}^{(\ell+1)}}\\
    &\qquad =  \Ee{\sum_{\mu\leq \ell} \lrr{\widehat{\eta}_{\mu}}^2 \frac{1}{n_{\ell+1}^2}\sum_{j_1,j_2=1}^{n_{\ell+1}} \sum_{k_1,k_2,k_3,k_4=1}^{n_\ell} W_{j_1k_1}^{(\ell+1)}W_{j_1k_2}^{(\ell+1)}W_{j_2k_3}^{(\ell+1)}W_{j_2k_4}^{(\ell+1)}\sigma(z_{k_1}^{(\ell)})\partial_\mu \sigma(z_{k_2}^{(\ell)})\sigma(z_{k_3}^{(\ell)})\partial_\mu \sigma(z_{k_4}^{(\ell)})}\\
    &\qquad =  \Ee{\sum_{\mu\leq \ell} \lrr{\widehat{\eta}_{\mu}}^2 \frac{4}{n_{\ell}^2}\sum_{k_1,k_2=1}^{n_{\ell}}  \sigma(z_{k_1}^{(\ell)})\partial_\mu \sigma(z_{k_1}^{(\ell)})\sigma(z_{k_2}^{(\ell)})\partial_\mu \sigma(z_{k_2}^{(\ell)}) + \frac{2}{n_{\ell+1}}\lrr{\sigma(z_{k_1}^{(\ell)})\partial_\mu \sigma(z_{k_2}^{(\ell)})}^2}\\
    &\qquad =\Y{\ell}{z\partial_\mu z, z\partial_\mu z} + \frac{1}{n_\ell}\lrr{1+\frac{2}{n_{\ell+1}}} \Y{\ell}{(z\partial_\mu z)^2} + \frac{2}{n_{\ell+1}} \Y{\ell}{z^2, \partial_\mu z^2}.
\end{align*}
Putting this together yields
\begin{align*}
    \Y{\ell+1}{z\partial_\mu z,z\partial_\mu z}&=  2\frac{\lrr{\eta^{(\ell+1)}}^2}{n_{\ell+1}}\Ee{\lrr{\frac{1}{n_\ell}||\sigma^{(\ell)}||^2}^2}\\
    &=\Y{\ell}{z\partial_\mu z, z\partial_\mu z} + \frac{1}{n_\ell}\lrr{1+\frac{2}{n_{\ell+1}}} \Y{\ell}{(z\partial_\mu z)^2} + \frac{2}{n_{\ell+1}} \Y{\ell}{z^2, \partial_\mu z^2},
\end{align*}
which is precisely \eqref{eq:Ymu-rec3}. Finally, we derive \eqref{eq:Ymu-rec4}. We have 
\begin{align*}
     \Y{\ell+1}{(z\partial_\mu z)^2}&= \Ee{\sum_{\mu\in \ell+1}\lrr{\widehat{\eta}_\mu}^2 \frac{1}{n_{\ell+1}}\sum_{j=1}^{n_{\ell+1}}\lrr{z_j^{(\ell+1)}\partial_\mu z_j^{(\ell+1)}}^2}\\
     &+\Ee{\sum_{\mu\leq \ell}\lrr{\widehat{\eta}_\mu}^2 \frac{1}{n_{\ell+1}}\sum_{j=1}^{n_{\ell+1}}\lrr{z_j^{(\ell+1)}\partial_\mu z_j^{(\ell+1)}}^2}.
\end{align*}
The first term is
\begin{align*}
    \Ee{\sum_{\mu\in \ell+1}\lrr{\widehat{\eta}_\mu}^2 \frac{1}{n_{\ell+1}}\sum_{j=1}^{n_{\ell+1}}\lrr{z_j^{(\ell+1)}\partial_\mu z_j^{(\ell+1)}}^2}&=\lrr{\eta^{(\ell+1)}}^2\Ee{\frac{1}{n_\ell}\sum_{j_1=1}^{n_\ell}\sigma(z_{j_1}^{(\ell)})  \frac{1}{n_{\ell+1}}\sum_{j=1}^{n_{\ell+1}}\lrr{z_j^{(\ell+1)}}^2}\\
     &=2\lrr{\eta^{(\ell+1)}}^2\Ee{\lrr{\frac{1}{n_\ell}||\sigma^{(\ell)}||^2}^2}.
\end{align*}
The second term is
\begin{align*}
    &\Ee{\sum_{\mu\leq \ell}\lrr{\widehat{\eta}_\mu}^2 \frac{1}{n_{\ell+1}}\sum_{j=1}^{n_{\ell+1}}\lrr{z_j^{(\ell+1)}\partial_\mu z_j^{(\ell+1)}}^2}\\
    &\qquad = \Ee{\sum_{\mu\leq \ell}\lrr{\widehat{\eta}_\mu}^2 \frac{1}{n_{\ell+1}}\sum_{j=1}^{n_{\ell+1}}\sum_{k_1,k_2,k_3,k_4=1}^{n_\ell}W_{jk_1}^{(\ell+1)}W_{jk_2}^{(\ell+1)}W_{jk_3}^{(\ell+1)}W_{jk_4}^{(\ell+1)}\sigma(z_{k_1}^{(\ell)})\sigma(z_{k_2}^{(\ell)})\partial_\mu \sigma(z_{k_3}^{(\ell)})\partial_\mu \sigma(z_{k_4}^{(\ell)})}\\
    &\qquad =\Ee{\sum_{\mu\leq \ell}\lrr{\widehat{\eta}_\mu}^2 \frac{4}{n_\ell}\sum_{k_1,k_2=1}^{n_\ell}\lrr{\sigma(z_{k_1}^{(\ell)})\partial_\mu \sigma(z_{k_2}^{(\ell)})}^2+2\sigma(z_{k_1}^{(\ell)})\partial_\mu \sigma(z_{k_1}^{(\ell)})\sigma(z_{k_2}^{(\ell)})\partial_\mu \sigma(z_{k_2}^{(\ell)})}\\
    &\qquad = 2\Y{\ell}{z\partial_\mu z, z\partial_\mu z} + \Y{\ell}{z^2,\partial_\mu z^2} + \frac{3}{n_\ell} \Y{\ell}{(z\partial_\mu z)^2}.
\end{align*}
So all together this yields
\begin{align*}
\Y{\ell+1}{(z\partial_\mu z)^2}&= 2\lrr{\eta^{(\ell+1)}}^2\Ee{\lrr{\frac{1}{n_\ell}||\sigma^{(\ell)}||^2}^2}\\
&+2\Y{\ell}{z\partial_\mu z, z\partial_\mu z} + \Y{\ell}{z^2,\partial_\mu z^2} + \frac{3}{n_\ell} \Y{\ell}{(z\partial_\mu z)^2},
\end{align*}
which is precisely \eqref{eq:Ymu-rec4}.
\end{proof}

Inspecting the recursions in Lemma \ref{lem:Ymu-recs} immediately shows that, for $\ell=1,\ldots, L$, we have that as $n\gives \infty$
\[
\Y{\ell}{(\partial_\mu z)^2}, \Y{\ell}{(\partial_\mu z)^2,z^2}, \Y{\ell}{(z\partial_\mu z)^2}  = O(1),\qquad \Y{\ell}{z\partial_\mu z,z\partial_\mu z} = O(n^{-1}).
\]
Thus, we obtain 
\begin{corollary}\label{cor:Ymu-form}
for $\ell=1,\ldots, L$, we have that 
\begin{align*}
     \Y{\ell+1}{(\partial_\mu z)^2} &= \frac{1}{2n_0}\norm{x}_2^2\sum_{\ell'=0}^\ell \lrr{\eta^{(\ell'+1)}}^2\\
     \Y{\ell+1}{(\partial_\mu z)^2,z^2} &= \frac{1}{2}\sum_{\ell'=0}^\ell \lrr{\eta^{(\ell'+1)}}^2\Ee{\lrr{\frac{1}{n_{\ell'}}||z^{(\ell')}||_2^2}^2} + O(n^{-1}).
\end{align*}
In particular, specializing to the case where $\eta^{(\ell)}=\eta$ is independent of $\ell$, we obtain 
\begin{align*}
     \Y{\ell+1}{(\partial_\mu z)^2} &= \frac{\ell+1}{2n_0}\norm{x}_2^2\eta^2\\
     \Y{\ell+1}{(\partial_\mu z)^2,z^2} &= \frac{\eta^2}{2}\lrr{\frac{||x||_2^2}{n_0}}^2\sum_{\ell'=0}^\ell  \prod_{\ell''=1}^{\ell'} \lrr{1+\frac{2}{n_{\ell''}}}+ O(n^{-1}).
\end{align*}
\end{corollary}
A simple consequence of this corollary is that, dropping terms on the order of $O(n^{-1}), O(\ell^{-1})$ and assuming that $\norm{x}^2=n_0$, gives
\begin{align*}
     \Y{\ell+1}{(\partial_\mu z)^2},\, \Y{\ell+1}{(\partial_\mu z)^2,z^2} &= \frac{1}{2}\ell\eta^2.
\end{align*}
Our next step is to obtain obtain and solve recursions for $Y$'s appearing in Lemma \ref{lem:H-rec} that involve sums over two network weights $\mu$ and $\nu$. The recursions are as follows.
\begin{lemma}\label{lem:Ymunu-recs}
We have
\begin{align*}
    \Y{\ell+1}{(\partial_\mu z)^2, (\partial_\nu z)^2}&=\lrr{\eta^{(\ell+1)}}^4\Ee{\lrr{\frac{1}{n_\ell}||\sigma(z^{(\ell)})||_2^2}^2}+\lrr{\eta^{(\ell+1)}}^2\left[ \Y{\ell}{(\partial_\mu z)^2,z^2}+\frac{1}{n_\ell}\Y{\ell}{(z\partial_\mu z)^2}\right]\\
    &+\Y{\ell}{(\partial_\mu z)^2,(\partial_\nu z)^2} + \frac{2}{n_{\ell+1}}\Y{\ell}{\partial_\mu z \partial_\nu z,\partial_\mu z \partial_\nu z}+\frac{1}{n_{\ell}}\lrr{1+\frac{2}{ n_{\ell+1}}}\Y{\ell}{(\partial_\mu z \partial_\nu z)^2}.\\
    \Y{\ell+1}{\partial_\mu z\partial_\nu z, \partial_\mu z\partial_\nu z}&=\lrr{\eta^{(\ell+1)}}^4\Ee{\lrr{\frac{1}{n_\ell}||\sigma^{(\ell)}||_2^2}^2}+ \frac{\lrr{\eta^{(\ell+1)}}^2}{n_{\ell+1}}\left[\Y{\ell}{(\partial_\mu z)^2,z^2}+\frac{1}{n_{\ell}}\Y{\ell}{(z\partial_\mu z)^2}\right]\\
    &+\Y{\ell}{\partial_\mu z\partial_\nu z, \partial_\mu z\partial_\nu z} + \frac{2}{n_{\ell+1}}\Y{\ell}{\lrr{\partial_\mu z}^2, \lrr{\partial_\nu z}^2} + \lrr{1+\frac{2}{ n_{\ell+1}}}\Y{\ell}{(\partial_\mu z\partial_\nu z)^2}\\
    \Y{\ell+1}{z\partial_{\mu\nu} z, \partial_\mu z\partial_\nu z}&= \frac{\lrr{\eta^{(\ell+1)}}^2}{n_{\ell+1}}\left[\Y{\ell}{z\partial_\mu z,z\partial_\mu z}+\frac{1}{n_\ell}\Y{\ell}{(z\partial_\mu z)^2}\right]\\
    &+\Y{\ell}{z\partial_{\mu\nu} z, \partial_\mu z\partial_\nu z}+ \frac{2}{n_{\ell+1}}\Y{\ell}{\partial_{\mu\nu}z\partial_\mu z,z\partial_\nu z}+\lrr{1+\frac{2}{ n_{\ell+1}}}\Y{\ell}{z\partial_{\mu}z\partial_\nu z \partial_{\mu\nu }z}\\
    \Y{\ell+1}{z\partial_{\mu} z, \partial_\mu z\partial_{\mu\nu} z}&= \Y{\ell}{z\partial_{\mu} z, \partial_\mu z\partial_{\mu\nu} z}+ \frac{2}{n_{\ell+1}}\Y{\ell}{z\partial_{\mu\nu} z,\lrr{z\partial_\mu z}^2}+\lrr{1+\frac{2}{ n_{\ell+1}}}\Y{\ell}{z\partial_\mu \partial_\nu \partial_{\mu\nu }z}\\
    \Y{\ell+1}{(\partial_{\mu\nu}z)^2,z^2}&=\lrr{\eta^{(\ell+1)}}^2 \left[\Y{\ell}{(\partial_{\mu}z)^2,z^2}+\frac{1}{n_\ell}\Y{\ell}{(z\partial_\mu z)^2}\right]\\
    &+\Y{\ell}{(\partial_{\mu\nu}z)^2,z^2}+\frac{2}{n_{\ell+1}}\Y{\ell}{z\partial_{\mu\nu}z,z\partial_{\mu\nu}z} + \frac{1}{n_\ell}\lrr{1+\frac{2}{n_{\ell+1}}}\Y{\ell}{(z\partial_{\mu\nu}z)^2}\\
    \Y{\ell+1}{(z\partial_{\mu\nu}z)^2}&=2\lrr{\eta^{(\ell+1)}}^2\left[\Y{\ell}{z^2,(\partial_\mu z)^2}+\frac{1}{n_\ell}\Y{\ell}{(z\partial_\mu z)^2}\right]\\
    &+\Y{\ell}{z^2, (\partial_{\mu\nu}z)^2} + 2\Y{\ell}{z\partial_{\mu\nu}z,z\partial_{\mu\nu}z}+\frac{2}{n_\ell} \Y{\ell}{(z\partial_{\mu\nu}z)^2}\\
    \Y{\ell+1}{(\partial_\mu z\partial_\nu z)^2}&=\lrr{\eta^{(\ell+1)}}^4 \Ee{\lrr{\frac{1}{n_\ell}||\sigma^{(\ell)}||^2}^2} + 2\lrr{\eta^{(\ell+1)}}^2 \left[\Y{\ell}{z^2, (\partial_\mu z)^2}+\frac{1}{n_\ell}\Y{\ell}{(z\partial_\mu z)^2}\right]\\
    &+\Y{\ell}{(\partial_\mu z)^2,(\partial_\nu z)^2} +2\Y{\ell}{\partial_\mu z \partial_\nu z,\partial_\mu z \partial_\nu z}+ \frac{1}{n_\ell}\Y{\ell}{(\partial_\mu z\partial_\nu z)^2}\\
    \Y{\ell+1}{z\partial_{\mu\nu}z,z\partial_{\mu\nu}z}&=\frac{\lrr{\eta^{(\ell+1)}}^2}{n_{\ell+1}}\lrr{\Y{\ell}{z^2, (\partial_\mu z)^2}+\frac{1}{n_\ell}\Y{\ell}{(z\partial_\mu z)^2}}+\lrr{1+\frac{1}{n_{\ell+1}}}\Y{\ell}{z\partial_{\mu\nu}z,z\partial_{\mu\nu}z} \\&+ \frac{1}{n_{\ell+1}}\Y{\ell}{z^2, (\partial_\mu z)^2}+\frac{1}{n_\ell}\lrr{1+\frac{1}{n_{\ell+1}}}\Y{\ell}{(z\partial_{\mu\nu}z)^2}\\
    \Y{\ell+1}{(\partial_{\mu\nu}z)^2}&=\lrr{\eta^{(\ell+1)}}^2\Y{\ell}{(\partial_\mu z)^2}+\Y{\ell}{(\partial_{\mu\nu}z)^2}\\
    \Y{\ell+1}{z\partial_{\mu\nu}z\partial_\mu z\partial_\nu z }&=\lrr{\eta^{(\ell+1)}}^2\lrr{\Y{\ell}{z\partial_\mu z, z\partial_\mu z}+\frac{1}{n_\ell}\Y{\ell}{(z\partial_\mu z)^2}}\\
    &+\Y{\ell}{z\partial_{\mu\nu}z, \partial_\mu z\partial_\nu z}+2\Y{\ell}{z\partial_\mu z, \partial_\nu z\partial_{\mu \nu}z}+\frac{3}{n_\ell}\Y{\ell}{z\partial_{\mu \nu z}\partial_\mu z \partial_\nu z}
\end{align*}
\end{lemma}
\begin{proof}
The proof of Lemma \ref{lem:Ymunu-recs} is very similar to that of Lemma \ref{lem:Ymu-recs}, so we will only give the details for $\Y{\ell+1}{(\partial_\mu z)^2,(\partial_\nu z)^2}$. We have
\begin{align*}
\Y{\ell+1}{(\partial_\mu z)^2,(\partial_\nu z)^2} &=\Ee{\sum_{\mu,\nu \in \ell+1} \lrr{\widehat{\eta}_\mu\widehat{\eta}_\nu}^2\frac{1}{n_{\ell+1}^2}\sum_{j_1,j_2=1}^{n_{\ell+1}}\lrr{\partial_\mu z_{j_1}^{(\ell+1)} \partial_\nu z_{j_2}^{(\ell+1)}}^2}\\
&+2\Ee{\sum_{\mu\leq \ell,\nu \in \ell+1} \lrr{\widehat{\eta}_\mu\widehat{\eta}_\nu}^2\frac{1}{n_{\ell+1}^2}\sum_{j_1,j_2=1}^{n_{\ell+1}}\lrr{\partial_\mu z_{j_1}^{(\ell+1)} \partial_\nu z_{j_2}^{(\ell+1)}}^2}\\
&+\Ee{\sum_{\mu,\nu \leq \ell} \lrr{\widehat{\eta}_\mu\widehat{\eta}_\nu}^2\frac{1}{n_{\ell+1}^2}\sum_{j_1,j_2=1}^{n_{\ell+1}}\lrr{\partial_\mu z_{j_1}^{(\ell+1)} \partial_\nu z_{j_2}^{(\ell+1)}}^2}.
\end{align*}
The first term equals
\[
\lrr{\eta^{(\ell+1)}}^4 \Ee{\lrr{\frac{1}{n_\ell}\norm{\sigma(z^{(\ell)})}^2}^2}.
\]
The second term is 
\begin{align*}
    &2\Ee{\sum_{\mu\leq \ell,\nu \in \ell+1} \lrr{\widehat{\eta}_\mu\widehat{\eta}_\nu}^2\frac{1}{n_{\ell+1}^2}\sum_{j_1,j_2=1}^{n_{\ell+1}}\lrr{\partial_\mu z_{j_1}^{(\ell+1)}\partial_\nu z_{j_1}^{(\ell+1)}}^2}\\
    &\qquad =2\lrr{\eta^{(\ell+1)}}^2\Ee{\sum_{\mu\leq \ell} \lrr{\widehat{\eta}_\mu}^2\frac{1}{n_{\ell}}\sum_{j_1=1}^{n_\ell}\lrr{\sigma(z_{j_1}^{(\ell)})}^2\frac{1}{n_{\ell+1}}\sum_{j_2=1}^{n_{\ell+1}}\lrr{\partial_\mu z_{j_2}^{(\ell+1)}}^2}\\
    &\qquad =2\lrr{\eta^{(\ell+1)}}^2\Ee{\sum_{\mu\leq \ell} \lrr{\widehat{\eta}_\mu}^2\frac{1}{n_{\ell}}\sum_{j_1=1}^{n_\ell^2}\lrr{\sigma(z_{j_1}^{(\ell)})}^2\lrr{\partial_\mu \sigma(z_{j_2}^{(\ell)})}^2}\\
    &\qquad =2\lrr{\eta^{(\ell+1)}}^2\left[\Y{\ell}{z^2,(\partial_\mu z)^2}+\frac{1}{n_\ell}\Y{\ell}{(z\partial_\mu z)^2}\right].
\end{align*}
Finally, the third term equals
{\small \begin{align*}
    &\Ee{\sum_{\mu,\nu \leq \ell} \lrr{\widehat{\eta}_\mu\widehat{\eta}_\nu}^2\frac{1}{n_{\ell+1}^2}\sum_{j_1,j_2=1}^{n_{\ell+1}}\lrr{\partial_\mu z_{j_1}^{(\ell+1)} \partial_\nu z_{j_2}^{(\ell+1)}}^2}\\
    &= \Ee{\sum_{\mu,\nu \leq \ell} \lrr{\widehat{\eta}_\mu\widehat{\eta}_\nu}^2\frac{1}{n_{\ell+1}^2}\sum_{j_1,j_2=1}^{n_{\ell+1}}\sum_{k_1,k_2,k_3,k_4=1}^{n_\ell}W_{j_1k_1}^{(\ell+1)}W_{j_1k_2}^{(\ell+1)}W_{j_2k_3}^{(\ell+1)}W_{j_2k_4}^{(\ell+1)}\partial_\mu \sigma(z_{k_1})^{(\ell)}\partial_\mu \sigma(z_{k_2})^{(\ell)} \partial_\nu \sigma(z_{k_3})^{(\ell)}\partial_\nu \sigma(z_{k_4})^{(\ell)}}\\
    &= \Ee{\sum_{\mu,\nu \leq \ell} \lrr{\widehat{\eta}_\mu\widehat{\eta}_\nu}^2\frac{4}{n_{\ell}^2}\sum_{j_1,j_2=1}^{n_{\ell}}\lrr{\partial_\mu \sigma(z_{j_1})^{(\ell)}\partial_\mu \sigma(z_{j_2})^{(\ell)}}^2+\frac{2}{n_{\ell+1}}\partial_\mu \sigma(z_{j_1})^{(\ell)}\partial_\nu \sigma(z_{j_1})^{(\ell)}\partial_\mu \sigma(z_{j_2})^{(\ell)}\partial_\nu \sigma(z_{j_2})^{(\ell)}}\\
    &=\Y{\ell}{(\partial_\mu z)^2,(\partial_\nu z)^2} + \frac{2}{n_{\ell+1}}\Y{\ell}{\partial_\mu z \partial_\nu z,\partial_\mu z \partial_\nu z}+\frac{3}{n_\ell n_{\ell+1}}\Y{\ell}{(\partial_\mu z \partial_\nu z)^2}.
\end{align*}}
Combining the preceding expressions completes the derivation of the recursion for $\Y{\ell}{(\partial_\mu z)^2,(\partial_\nu z)^2}$.
\end{proof}
Inspecting these recursions immediately shows the following
\begin{corollary}\label{cor:Ymunu-form}
For $\ell=1,\ldots, L$, we have that as $n\gives \infty$
\begin{align*}
   \Y{\ell}{\partial_\mu z\partial_\nu z, \partial_\mu z\partial_\nu z},\, \Y{\ell}{z\partial_{\mu} z, \partial_\mu z\partial_{\mu\nu} z},\,  \Y{\ell}{z\partial_{\mu\nu} z, \partial_\mu z\partial_\nu z},\, \Y{\ell}{z\partial_{\mu\nu z},z\partial_{\mu\nu z}} = O(n^{-1}),
\end{align*}
while all the other $Y$'s are order $1$. Moreover, 
\begin{align*}
    \Y{\ell+1}{(\partial_\mu z)^2, (\partial_\nu z)^2}&=\lrr{\eta^{(\ell+1)}}^4\Ee{\lrr{\frac{1}{2n_\ell}||z^{(\ell)}||_2^2}^2}+\lrr{\eta^{(\ell+1)}}^2 \Y{\ell}{(\partial_\mu z)^2,z^2}\\
    &+\Y{\ell}{(\partial_\mu z)^2, (\partial_\nu z)^2} + O(n^{-1})
\end{align*}
Hence, assuming that $\norm{x}^2=n_0$, we find
\begin{align*}
    \Y{\ell+1}{(\partial_\mu z)^2, (\partial_\nu z)^2}&=\frac{1}{4}\eta^4+\frac{1}{2}\eta^2\ell +\Y{\ell}{(\partial_\mu z)^2, (\partial_\nu z)^2} + O(n^{-1})=\frac{\eta^2\ell^2}{4} + O(\ell^{-1})+O(n^{-1}).
\end{align*}
Similarly, 
\begin{align*}
    \Y{\ell+1}{(\partial_{\mu\nu} z)^2}&= \ell \eta^4 + \Y{\ell}{(\partial_{\mu\nu} z)^2} + O(n^{-1}) = C_3 \ell^2 \eta^4 + O(n^{-1}).
\end{align*}
And also,
\begin{align*}
    \Y{\ell+1}{z^2, (\partial_{\mu\nu} z)^2}&= \ell \eta^4 + \Y{\ell}{(\partial_{\mu\nu} z)^2} + O(n^{-1}) = C_4 \ell^2 \eta^4 + O(n^{-1}).
\end{align*}
\end{corollary}
\vspace{.2cm}
\subsection{Completion of the Proof of Theorem \ref{thm:HS}}
We are now in a position to complete the proof of Theorem \ref{thm:HS}. By combining the results at the end of the previous section with Corollary \ref{cor:H-Y}, we find
\begin{align*}
    \Ee{\norm{H_{\text{eff}}}_{HS}^2}&=2\lrr{\eta^{(\ell+1)}}^4\Ee{\frac{1}{n_\ell}\norm{\sigma^{(L)}}^2}^2\\
    &+\lrr{\eta^{(L+1)}}^2\lrr{2\Y{L}{z^2, (\partial_\mu z)^2}+\Y{L}{(\partial_\mu z)^2}}\\
    &+\Y{L}{z^2,(\partial_{\mu\nu }z)^2} +\Y{L}{(\partial_{\mu\nu}z)^2}+\Y{L}{(\partial_\mu z)^2,(\partial_\nu z)^2} + O(n^{-1})\\
    &=C_1\eta^4\lrr{\frac{1}{n_0}\norm{x}^2}^2+C_2\eta^4L\\
    &+\Y{L}{z^2,(\partial_{\mu\nu }z)^2} +\Y{L}{(\partial_{\mu\nu}z)^2}+\Y{L}{(\partial_\mu z)^2,(\partial_\nu z)^2} + O(n^{-1})
\end{align*}
for some universal constants $C_1,C_2$. Moreover, we also find
\begin{align*}
    \Y{\ell+1}{(\partial_\mu z)^2, (\partial_\nu z)^2}&=\lrr{\eta^{(\ell+1)}}^4\Ee{\lrr{\frac{1}{2n_\ell}||z^{(\ell)}||_2^2}^2}+\lrr{\eta^{(\ell+1)}}^2 \Y{\ell}{(\partial_\mu z)^2,z^2}\\
    &+\Y{\ell}{(\partial_\mu z)^2, (\partial_\nu z)^2} + O(n^{-1}).
\end{align*}
Specializing to the case where $\eta^{(\ell)}=\eta$ is independent of $\ell$, dropping terms on the order of $O(n^{-1}), O(\ell^{-1})$, and assuming that $\norm{x}^2=n_0$, we find
\begin{align*}
    \Y{\ell+1}{(\partial_\mu z)^2, (\partial_\nu z)^2}&=\frac{1}{4}\eta^4+\frac{1}{2}\eta^2\ell +\Y{\ell}{(\partial_\mu z)^2, (\partial_\nu z)^2} + O(n^{-1})=\frac{\eta^2\ell^2}{4} + O(\ell^{-1})+O(n^{-1}).
\end{align*}
Similarly, 
\begin{align*}
    \Y{\ell+1}{(\partial_{\mu\nu} z)^2}&= \ell \eta^4 + \Y{\ell}{(\partial_{\mu\nu} z)^2} + O(n^{-1}) = C_3 \ell^2 \eta^4 + O(n^{-1}).
\end{align*}
And also,
\begin{align*}
    \Y{\ell+1}{z^2, (\partial_{\mu\nu} z)^2}&= \ell \eta^4 + \Y{\ell}{(\partial_{\mu\nu} z)^2} + O(n^{-1}) = C_4 \ell^2 \eta^4 + O(n^{-1}).
\end{align*}
So all together, we obtain 
\[
\Ee{\norm{H_{\text{eff}}}_{HS}^2} = \frac{\eta^2L^2}{4}\lrr{1 + O(n^{-1})+O(L^{-1})}.
\]
In the setting of Theorem \ref{thm:HS} we have
\[
\Ee{\norm{H^{(L+1)}}_{HS}^2}  =  \Ee{\norm{H_{\text{eff}}}_{HS}^2}n^2\lrr{1+O(n^{-1})},
\]
up to universal constants, which yields the stated estimate.

\end{document}